\documentclass[10pt, letterpaper, twoside]{article}

\usepackage{amsmath,amsthm,amssymb,mathrsfs} 
\usepackage[titletoc,title]{appendix}
\usepackage[numbers,sort&compress]{natbib}
\usepackage[hidelinks]{hyperref}
\usepackage{cleveref}

\usepackage{chngcntr}
\usepackage{titlesec}

\usepackage{lipsum}
\usepackage[hang,flushmargin]{footmisc}

\allowdisplaybreaks

\newcommand\blfootnote[1]{%
  \begingroup
  \renewcommand\thefootnote{}\footnote{#1}%
  \addtocounter{footnote}{-1}%
  \endgroup
}

\newtheorem{theorem}{Theorem}[section]
\newtheorem{assumption}{Assumption}[section]
\newtheorem{lemma}{Lemma}[section]
\newtheorem{prop}{Proposition}[section]

\newtheorem{definition}{Definition}[section]
\newtheorem{remark}{Remark}[section]
\newtheorem{example}{Example}[section]

\numberwithin{equation}{section}

\usepackage{enumitem}
\setlist[itemize]{itemsep=0pt,parsep=2pt,topsep=2pt}
\setlist[enumerate]{itemsep=0pt,parsep=2pt,topsep=2pt}
\def\E{\mathbb{E}}
\def\R{\mathbb{R}}
\def\I{\mathcal{I}}
\def\F{\mathcal{F}}

\def\tl{\tilde}
\DeclareMathAlphabet{\mathbbb}{U}{bbold}{m}{n}
\newcommand{\ind}{\mathbbb{1}}
\def\={\mathop{:=}}
\def\argmax{\mathop{\arg\max}}
\def\Q{\mathcal{Q}}
\def\1{\mathbf{1}}

\def\myqed{\qed}

\def\S{\mathcal{S}}
\def\A{\mathcal{A}}
\def\P{\mathcal{P}}
\def\T{\mathcal{T}}

\def\tm{t_{\text{\rm min}}}

\begin{document} 
\markboth{Yu, Wan, \& Sutton}{RVI-based Average-Reward RL in SMDPs}

\title{Average-reward reinforcement learning in semi-{Markov} decision processes via relative value iteration%
\footnote{This research was supported in part by DeepMind, Amii, and the Natural Sciences and Engineering Research Council of Canada (NSERC) under grants RGPIN-2024-04939 and DGECR-2024-00312.}}

\author{{\normalsize Huizhen Yu\textsuperscript{a},
Yi Wan\textsuperscript{a}, 
and Richard S. Sutton\textsuperscript{a,b}}}

\date{} 
\maketitle

\blfootnote{\textsuperscript{a}Department of Computing Science, University of Alberta, Canada}
\blfootnote{\textsuperscript{b}Alberta Machine Intelligence Institute (Amii), Canada}

\blfootnote{{Emails:}\,\texttt{janey.hzyu@gmail.com} (HY,\! corresponding author); \texttt{wan6@ualberta.ca} (YW); \texttt{rsutton@ualberta.ca} (RS)}

\vspace*{-0.8cm}

\noindent{\bf Abstract:} 
This paper applies the authors' recent results on asynchronous stochastic approximation (SA) in the Borkar--Meyn framework to reinforcement learning in average-reward semi-Markov decision processes (SMDPs). We establish the convergence of an asynchronous SA analogue of Schweitzer's classical relative value iteration algorithm, RVI Q-learning, for finite-space, weakly communicating SMDPs. In particular, we show that the algorithm converges almost surely to a compact, connected subset of solutions to the average-reward optimality equation, with convergence to a unique, sample path-dependent solution under additional stepsize and asynchrony conditions. Moreover, to make full use of the SA framework, we introduce new monotonicity conditions for estimating the optimal reward rate in RVI Q-learning. These conditions substantially expand the previously considered algorithmic framework and are addressed through novel arguments in the stability and convergence analysis of RVI Q-learning.



\section{Introduction}

Markov and semi-Markov decision processes (MDPs/SMDPs) are widely used mathematical models for sequential decision-making under uncertainty in discrete or continuous time \cite{Put94}, with applications in robotics, finance, operations research, and beyond. This paper focuses on model-free reinforcement learning (RL) methods for solving MDPs and SMDPs with finite state and action spaces under an average-reward criterion, where the goal is to optimize sustained long-term performance. RL offers a powerful approach for tackling large, complex problems in real-world scenarios. Underpinning many RL algorithms---and essential for ensuring their reliability---is the theory of asynchronous stochastic approximation (SA). In \cite{YWS25a}, we developed general stability and convergence results for asynchronous SA within the Borkar--Meyn framework \cite{Bor98,BoM00}. In this paper, we apply those results to study and extend a family of learning algorithms based on relative value iteration (RVI), a classical method for solving average-reward problems \cite{Sch71,ScF77,Whi63}. We focus on SMDPs, which include MDPs as special cases.

Building on this SA foundation, our work extends the seminal contributions of Abounadi, Bertsekas, and Borkar \cite{ABB01}, who first applied the Borkar–Meyn framework \cite{Bor98,BoM00} in developing an asynchronous stochastic RVI algorithm, RVI Q-learning, for finite state and action MDPs. Their work combined asynchronous SA theory with the convergence results of Borkar and Soumyanath \cite{BoS97}---which concern continuous-time analogs of fixed-point iteration involving nonexpansive mappings---in the design and analysis of a `self-regulating' RVI Q-Iearning algorithm, laying the groundwork for subsequent developments.

Also closely related to our work are the recent contributions of Wan, Naik, and Sutton \cite{WNS21a,WNS21b}, who broadened the algorithmic framework and extended RVI Q-learning to hierarchical control in average-reward MDPs. In particular, an immediate predecessor of the algorithm we develop in this paper is an algorithm proposed in \cite{WNS21b} for solving average-reward SMDPs arising from those hierarchical control settings. Like their algorithm, our new RVI Q-learning algorithm serves as an asynchronous SA analogue of Schweitzer's classical RVI algorithm for SMDPs~\cite{Sch71}.

The prior analyses of RVI Q-learning \cite{ABB01,WNS21a,WNS21b} apply primarily to unichain MDPs or SMDPs, where the average-reward optimality equation (AOE) admits a unique solution (up to an additive constant), to which RVI Q-learning was shown to converge almost surely (a.s.). Our recent work \cite{WYS24} extended these analyses to weakly communicating MDPs or SMDPs (arising from hierarchical control settings), where the AOE solutions generally do not possess such a uniqueness property, significantly relaxing the model assumptions under which RVI Q-learning can be applied.

Another limitation of the prior studies \cite{ABB01,WNS21a,WNS21b} is the presence of theoretical gaps in their stability proofs---an essential step for establishing convergence. Although these studies showed that the Borkar--Meyn stability criterion \cite{BoM00} is satisfied by the associated update functions, for the stability of the asynchronous algorithms, the original study \cite{ABB01} relied on an unproven assertion in \cite[Thm.\ 2.5]{BoM00}, while the later works \cite{WNS21a,WNS21b} argued incorrectly (see \cite[Sec.\ 1.1]{YWS25a} for a more detailed discussion). These issues were resolved in our companion paper \cite{YWS25a}, which established general stability and convergence results for asynchronous SA, providing a strengthened theoretical foundation for RVI Q-learning.

We applied our asynchronous SA results from \cite{YWS25a} in \cite{WYS24} to analyze existing RVI Q-learning algorithms from \cite{ABB01,WNS21a,WNS21b} for weakly communicating MDPs/SMDPs. One purpose of this paper is to further demonstrate the applicability of our SA results by developing a new, more general RVI Q-learning algorithm for SMDPs.

The formulation and analysis of this new algorithm constitute a central contribution of this paper. They are in part motivated by our desire to fully exploit the implications of the Borkar--Meyn stability criterion \cite{BoM00} and our asynchronous SA results from \cite{YWS25a}. As noted, the new RVI Q-learning algorithm---like its predecessor by Wan et al.~\cite{WNS21b}---is an asynchronous SA analogue of Schweitzer's classical RVI algorithm for SMDPs~\cite{Sch71}. The main innovative feature introduced here is a set of new monotonicity conditions for estimating the optimal reward rate in RVI Q-learning, specifically, strict monotonicity under scalar translation (see Assum.\ \ref{cond-f} and Def.\ \ref{def-sistr}). These conditions significantly expand the algorithmic framework of RVI Q-learning previously considered in \cite{ABB01,WNS21a,WNS21b,WYS24}, and we address them with novel proof arguments in the stability and convergence analysis of RVI Q-learning (cf.\ Rem.~\ref{rmk-novel-proof}). 

A further contribution of this paper is a sharpened characterization of the convergence behavior of RVI Q-learning. As noted earlier, in weakly communicating SMDPs, solutions to the AOE may have multiple degrees of freedom and need not be unique up to an additive constant \cite{ScF78}. In the absence of stochastic noise and asynchrony, the classical RVI algorithm converges to a single solution determined by the initial condition \cite{Sch71,ScF77,platzman1977improved}. It is therefore natural to ask whether RVI Q-learning can also converge to a unique solution, rather than merely approaching the solution set. 

We give an affirmative answer to this question, providing implementable algorithmic conditions under which RVI Q-learning converges to a unique (sample-path-dependent) solution. This constitutes another main contribution of the paper. The proof builds on our recent analysis of shadowing properties of asynchronous SA from our companion paper \cite{YWS25a}, which in turn relies on the dynamical system approach of Hirsch and Bena\"{i}m that has been applied previously to synchronous SA algorithms \cite{Ben96,Ben99,BeH96,Hir94}.  

Since MDPs are special cases of SMDPs, all of our analyses and results also apply to average-reward MDPs.

We now formally outline the main contributions of this paper:
\begin{itemize}[leftmargin=0.72cm,labelwidth=!]
\item[(i)] We introduce new monotonicity conditions to substantially generalize RVI Q-learning (see Assum.~\ref{cond-f}, Ex.~\ref{ex-f}, and Prop.~\ref{prp-sol-set}). 
\item[(ii)] Leveraging our SA results from \cite{YWS25a}, we establish the almost sure convergence of the generalized algorithm to a compact and connected subset of solutions to the AOE, for weakly communicating SMDPs under mild model assumptions (Thm.~\ref{thm-rvi-ql}).
\item[(iii)] Using our sharpened convergence result from \cite{YWS25a}---based on an analysis of shadowing properties of asynchronous SA---we further establish, under additional stepsize and asynchrony conditions, the almost sure convergence of RVI Q-learning to a unique AOE solution that depends on the sample path (Thm.~\ref{thm-rvi-ql2}).
\end{itemize}

We mention that \citet{YuH26} recently applied the RVI Q-learning results established here for weakly communicating MDPs/SMDPs, together with a hierarchical problem decomposition approach \cite{Bat73c,Sch84}, to develop model-free average-reward RL algorithms for general multichain problems.

\vspace*{0.08cm}
The paper is organized as follows. In Sec.~\ref{sec-sa}, we recall the general asynchronous SA framework and the results from \cite{YWS25a} that will be applied in this paper. In Sec.~\ref{sec-rvi}, after an overview of average-reward SMDPs (Sec.~\ref{sec-3.1}), we present our generalized RVI Q-learning algorithm and its convergence properties (Secs.~\ref{sec-3.2},~\ref{sec-3.3}). The proofs are given in Sec.~\ref{sec-4}. The paper concludes with a brief summary in Sec.~\ref{sec-conc-rmks}.

\vspace*{0.1cm}
\noindent {\bf Notation:} In this paper, $\ind \{ E\}$ denotes the indicator for an event $E$; $\P(X)$ denotes the set of probability measures on a measurable space $X$; and $\E[\,\cdot\,]$ denotes expectation. For $a, b \in \R$, $a \vee b \= \max \{ a, b\}$ and $a \wedge b \= \min \{ a, b\}$. The symbol $\1$ stands for the vector of all ones in $\R^d$. The addition $x + c$ for a vector $x \in \R^d$ and a scalar $c$ represents adding $c$ to each component of $x$. This notation also applies to the addition of a vector-valued function with a scalar-valued function. 
For convergence of functions, we write $\overset{p}{\to}$ for pointwise convergence and $\overset{u.c.}{\to}$ for uniform convergence on compact subsets of the domain.

\section{Asynchronous SA: Assumptions \& Key Results from \cite{YWS25a}} 
 \label{sec-sa}

This section recalls the asynchronous SA framework and two convergence theorems from our companion paper \cite{YWS25a} that will be applied in this paper. 

\subsection{Setup and Key Theorems} \label{sec-2.1}
We consider the asynchronous SA framework introduced by Borkar~\cite{Bor98}, with an associated function satisfying the Borkar--Meyn stability criterion \cite{BoM00}. Specifically, consider an SA algorithm that operates in a finite-dimensional space $\R^d$ and, given an initial vector $x_0 \in \R^d$, iteratively computes $x_n \in \R^d$ for $n \geq 1$ using an asynchronous scheme as follows. Let $\alpha_n > 0$, $n \geq 0$, be a given sequence of diminishing stepsizes. At iteration $n \geq 0$, a nonempty random subset $Y_n$ is selected from $\I : = \{ 1, 2, \ldots, d\}$. For $i \not\in Y_n$, the $i$th component of $x_n$ remains unchanged: $x_{n+1}(i) = x_n(i)$. For $i \in Y_n$, with $\nu(n,i) \= \sum_{k=0}^{n-1} \ind \{ i \in Y_k\}$ denoting the cumulative number of updates to the $i$th component prior to iteration $n$, let
\begin{equation} \label{eq-alg0}
    x_{n+1}(i)  = x_n(i)  + \alpha_{\nu(n,i)} \big( h_i (x_n) + M_{n+1}(i) + \epsilon_{n+1}(i) \big), \qquad i \in Y_n.
\end{equation}  
Here, $h_i$ denotes the $i$th component of a function $h : \R^d \to \R^d$, assumed to be Lipschitz continuous, and $M_{n+1}(i)$ and $\epsilon_{n+1}(i)$ denote the $i$th components of a centred noise term $M_{n+1}$ and a biased noise term $\epsilon_{n+1}$, respectively. 
The algorithm is associated with an increasing family $\{\F_n\}_{n \geq 0}$ of $\sigma$-fields, where $\F_n \supset \sigma (x_m, Y_m, M_m, \epsilon_m; m \leq n)$. It satisfies Assums.~\ref{cond-h}--\ref{cond-us} listed below. 

Assumption~\ref{cond-h} is the Borkar--Meyn stability criterion \cite{BoM00}, which is central to ensuring the stability of the algorithm---that is, the boundedness of the iterates $\{x_n\}$. The criterion is stated through an ordinary differential equation (ODE) associated with a scaling limit of the function $h$ and imposes conditions on the behavior of solutions to $\dot{x} = h(x)$ when they are far from the origin.

\begin{assumption}[Conditions on function $h$] \label{cond-h} \hfill 
\begin{enumerate}[leftmargin=0.75cm,labelwidth=!] 
\item[{\rm (i)}] $h$ is Lipschitz continuous: for some $L_h \geq 0$, $\| h(x) - h(y) \| \leq L_h \| x - y\|$ for all $x, y \in \R^{d}$.
\item[{\rm (ii)}] Define $h_c(x) \= h(cx)/c$ for $c \geq 1$. As $c \uparrow \infty$, $h_c \overset{u.c.}{\to} h_\infty : \R^d \to \R^d$.
\item[\rm (iii)] The ODE
$ \dot{x}(t) = h_\infty (x(t)) $
has the origin as its unique globally asymptotically stable equilibrium.
\end{enumerate}
\end{assumption} 

Assumption~\ref{cond-ns} on the noise terms is more general than those considered in \cite{Bor98,BoM00} and is needed to extend the analysis of RVI Q-learning from MDPs to SMDPs.
Assumptions~\ref{cond-ss}--\ref{cond-us} are algorithmic conditions based on those introduced in Borkar~\cite{Bor98} for asynchronous SA, and they closely follow the versions adopted in the RVI Q-learning algorithm of~\cite{ABB01}.
We will discuss these conditions in Sec.~\ref{sec-sa-cond} in connection with RVI Q-learning; see also the related remarks in our companion paper \cite[Sec.~2.1]{YWS25a}.

\begin{assumption}[Conditions on noise terms $M_n, \epsilon_n$] \label{cond-ns} For all $n \geq 0$, we have:
\begin{enumerate}[leftmargin=0.7cm,labelwidth=!] 
\item[{\rm (i)}] 
$\E [ \| M_{n+1} \| ] < \infty$, $\E [ M_{n+1} \mid \F_n ] = 0$ and $ \E[ \| M_{n+1} \|^2 \mid \F_n ] \leq K_n (1 +\| x_n \|^2)$ a.s., for some $\F_n$-measurable $K_n \geq 0$ with $\sup_n K_n < \infty$ a.s. 
\item[{\rm (ii)}] 
$\| \epsilon_{n+1} \| \leq \delta_{n+1} ( 1 + \| x_n \|)$, where $\delta_{n+1}$ is $\F_{n+1}$-measurable and $\delta_n \overset{n \to \infty}{\to} 0$ a.s. 
\end{enumerate}
\end{assumption} 

\begin{assumption}[Stepsize conditions]  \label{cond-ss} \hfill
\begin{enumerate}[leftmargin=0.7cm,labelwidth=!] 
\item[{\rm (i)}] $\sum_n \alpha_n = \infty$, $\sum_n \alpha_n^2 < \infty$, and 
$\alpha_{n+1} \leq \alpha_n$ for all $n$ sufficiently large.
\item[{\rm (ii)}] For $x \in (0,1)$, 
$\sup_n \frac{\alpha_{[ x n]}}{ \alpha_n} < \infty$,
where $[x n]$ denotes the integral part of $xn$.
\item[{\rm (iii)}] For $x \in (0,1)$, as $n \to \infty$, $\frac{ \sum_{k=0}^{[ y n ]} \alpha_k }{ \sum_{k=0}^{n} \alpha_k} \to 1$ uniformly in $y \in [x, 1]$.
\end{enumerate}
\end{assumption} 

For $x > 0$ and $n \geq 0$, define $N(n,x) \= \min  \left\{ m > n : \sum_{k = n}^m \alpha_k \geq x \right\}$.
\begin{assumption}[Asynchronous update conditions] \label{cond-us} \hfill 
\begin{enumerate}[leftmargin=0.7cm,labelwidth=!]
\item[{\rm (i)}]  For some deterministic $\Delta > 0$, 
$\liminf_{n \to \infty} \nu(n,i)/n  \geq \Delta$ a.s., for all $i \in \I$.
\item[\rm (ii)] For each $x > 0$, the limit $\lim_{n \to \infty} \frac{ \sum_{k = \nu(n,i)}^{\nu(N(n,x), i)} \alpha_k}{ \sum_{k = \nu(n,j)}^{\nu(N(n,x), j)} \alpha_k}$ exists a.s., for all $i, j \in \I$.
\end{enumerate}
\end{assumption} 

Theorem~\ref{thm-ql} below establishes the convergence of algorithm~\eqref{eq-alg0} under the preceding assumptions. It is derived in \cite{YWS25a} as a special case for our RL application from general stability and convergence results, which build on the ODE-based proof methods introduced in \cite{Bor98,BoM00} (see also \cite{Bor23}) and relate the algorithm's asymptotic behavior to solutions of associated (autonomous and non-autonomous) ODEs. In this theorem, the asymptotic behavior of algorithm~\eqref{eq-alg0} is described first in terms of the iterates $\{x_n\}$ and then in terms of a continuous trajectory formed from them, which provides additional insight into the algorithm's dynamic behavior. 

Before stating the theorem, we recall the relevant definitions. Consider the ODE $\dot{x}(t) = h(x(t))$ with equilibrium set
$E_h \= \{ x \in \R^d \mid h(x) = 0 \}.$
Under Assum.~\ref{cond-h}, $E_h$ is nonempty and compact \cite[Lem.~2.1]{YWS25a}. In the theorem, we will further assume that $E_h$ is \emph{globally asymptotically stable}. Recall that this means that (i) every solution $x(t)$ of the ODE converges to $E_h$ as $t \to \infty$, and (ii) $E_h$ is \emph{Lyapunov stable}, i.e., for every $\epsilon > 0$, there exists $\delta > 0$ such that solutions starting from the $\delta$-neighborhood of $E_h$ remain within its $\epsilon$-neighborhood for all $t \geq 0$ (see \cite[Chap.\ 4.2.2]{KuY03}). 
We will later show that this condition is satisfied in our generalized RVI Q-learning setting (see Prop.~\ref{prp-rvi-asyn-stab}). We note, however, that this condition can be restrictive in general SA applications, such as stochastic optimization, and that the general results and proof techniques in \cite{YWS25a} are not inherently limited to this setting.

Define a linearly interpolated trajectory $\bar x(t)$ from $\{x_n\}$ with aggregated stepsizes 
$$\textstyle{\tl \alpha_n = \sum_{i \in Y_n}  \alpha_{\nu(n, i)},} \quad n \geq 0,$$ 
as the elapsed times between consecutive iterates. Specifically, for $n \geq 0$, let $\tl t(n) \= \sum_{k=0}^{n -1} \tl \alpha_k$ with $\tl t(0) \= 0$, and define $\bar x(\tl t(n)) \= x_n$ and 
\begin{equation} 
 \bar x(t) \=  x_n +  \tfrac{t - \tl t(n)}{\tl t(n+1) - \tl t(n)} \, ( x_{n+1} - x_n), \ \ \,  t \in [\tl t(n), \tl t(n+1)]. \label{eq-cont-traj2}
\end{equation} 
Extend $\bar x(\cdot)$ to $(-\infty, \infty)$ by setting $\bar x(t) = x_0$ for $t < 0$. We refer to the time variable $t$ in $\bar x(t)$ as the `ODE-time.'

\begin{theorem}[{from \cite[Cor.~2.1]{YWS25a}}] \label{thm-ql} 
Let Assums.~\ref{cond-h}--\ref{cond-us} hold, and let $E_h$ be globally asymptotically stable for the ODE $\dot{x}(t) = h(x(t))$. Then the following hold a.s.\ for algorithm \eqref{eq-alg0}:
\begin{itemize}[leftmargin=0.65cm,labelwidth=!]
\item[\rm (i)] The sequence $\{x_n\}$ converges to a compact connected subset of $E_h$.
\item[\rm (ii)] For any $\delta > 0$ and any convergent subsequence $\{x_{n_k}\}$, as $k \to \infty$,
$$\tau_{\delta,k} \= \min \left\{ |s| :  \| \bar x(t_{n_k} + s) - x^* \| > \delta, \ s \in \R \right\} \to \infty,$$ 
where $\bar x(\cdot)$ is the continuous trajectory defined above, 
$t_{n_k} = \tl t(n_k)$ is the ODE-time when $x_{n_k}$ is generated, and $x^* \in E_h$ is the point to which $\{x_{n_k}\}$ converges.
\end{itemize}
\end{theorem}

This theorem will be applied to establish general convergence properties of RVI Q-learning (see Thm.~\ref{thm-rvi-ql} and Sec.~\ref{sec-4.1}). Both parts of the theorem describe convergence to the set $E_h$, rather than to a single point. When $E_h$ contains non-isolated equilibrium points, part (ii) shows that as time progresses, the duration of ODE-time that algorithm \eqref{eq-alg0} spends near each limiting point increases without bound.

Next, we recall a sharpened convergence theorem from \cite{YWS25a}, together with the additional conditions it imposes on the biased noise terms, stepsizes, and asynchrony, to ensure convergence to a single equilibrium point. 

\begin{assumption}[Additional condition on noise term $\epsilon_n$] \label{cond-mns}
There exists a deterministic constant $\mu_\delta < 0$ such that $\limsup_{n \to \infty} \tfrac{\ln \delta_{n+1}}{\sum_{k=0}^n \alpha_k} \leq \mu_\delta$ a.s., where $\{\delta_{n}\}$ are the random variables involved in Assum.~\ref{cond-ns}(ii) for $\{\epsilon_n\}$.
\end{assumption}

Consider two specific stepsize sequences that satisfy Assum.~\ref{cond-ss}:  
\begin{equation} \label{def-stepsize-class}
   \text{$\alpha_n = \tfrac{1}{A n}$ (class 1) \ \ \ and \ \ \ $\alpha_n = \tfrac{1}{A n \ln n}$ (class 2)},
\end{equation}   
where $A > 0$ is a scaling parameter (with $\alpha_n$ set to $\tfrac{1}{A}$ if the denominator is zero). For class-1 stepsizes, we impose an additional condition on asynchony:

\begin{assumption}[Additional asynchrony condition] \label{cond-mus} 
For class-1 stepsizes, the asynchronous update schedules are such that a.s.\ for all $i \in \I$, $\nu(n,i)/n \to p_i$ as $n \to \infty$, for some (sample-path-dependent) $p_i \in (0,1]$. Moreover, there exists a deterministic constant $\gamma > 0$ such that $\limsup_{n \to \infty} n^{\gamma} \big| \nu(n,i)/n - p_i \big| < \infty$ a.s.
\end{assumption} 

Let $L_h$ be the Lipschitz constant of $h$ under $\| \cdot\|_\infty$. 

\begin{theorem}[{from \cite[Thm.~2.3]{YWS25a}}] \label{thm-3}
Consider algorithm~\eqref{eq-alg0} with class-1 stepsizes where $\tfrac{A}{2} > L_h$ or class-2 stepsizes where $A >  L_h$. Let Assums.~\ref{cond-h}-\ref{cond-us},~\ref{cond-mns} with $\mu_\delta < - L_h$ hold, and for class-1 stepsizes, also assume Assum.~\ref{cond-mus} with $\gamma A > L_h$.
For the ODE $\dot{x}(t) = h(x(t))$, suppose that the set $E_h$ is globally asymptotically stable and that every solution $x(t)$ converges to a unique point in $E_h$ as $t \to \infty$. Then the sequence $\{x_n\}$ from algorithm \eqref{eq-alg0} converges a.s.\ to a point in $E_h$ that depends on the sample path. 
\end{theorem}

\begin{remark}[{\normalfont On shadowing analysis and local Lipschitz constants}] \label{rmk-local-lip} \rm
This theorem is based on an analysis of the shadowing properties of the trajectory $\bar x(\cdot)$ using a dynamical systems approach of Hirsch and Bena\"{i}m \cite{Ben96,Ben99,BeH96,Hir94}. The notion of shadowing concerns the asymptotic alignment, as $t \to \infty$, of the trajectory $\bar x(\cdot)$ with a \emph{unique solution trajectory} of the associated limiting ODE, and this alignment is key to ensuring the desired convergence to a single equilibrium point.

Furthermore, the shadowing analysis underlying this theorem allows the use of a local Lipschitz constant of $h$ on a neighborhood of the compact, globally asymptotically stable set $E_h$, instead of the global Lipschitz constant $L_h$ in the conditions of the theorem. This is because the shadowing analysis concerns asymptotic sample-path properties. The a.s.\ convergence of $\{x_n\}$ to $E_h$ already established in Thm.~\ref{thm-ql} ensures that the trajectory $\bar x(t)$ lies entirely in an arbitrarily small neighborhood of $E_h$ for all sufficiently large $t$, a.s. Thus, the analysis can be restricted to a suitable neighborhood of $E_h$ containing both the tail of the trajectory and the corresponding comparison ODE solution segments.

We will apply this theorem to strengthen the convergence analysis of RVI Q-learning for weakly communicating SMDPs (see Thm.~\ref{thm-rvi-ql2} and Sec.~\ref{sec-shad2}). The observation just mentioned will also be useful in motivating certain algorithmic choices for RVI Q-learning (cf.\ Rem.~\ref{rem-mdp}). \myqed
\end{remark}

\subsection{Remarks on the Required Conditions} \label{sec-sa-cond}

The algorithmic design of RVI Q-learning in this paper, as in prior studies \cite{ABB01,WNS21a,WNS21b,WYS24}, is shaped by the requirements of the underlying asynchronous SA framework, both directly and indirectly. Some algorithmic requirements, such as those on stepsizes and asynchrony, will be explicitly imposed on RVI Q-learning. Others, like the Borkar--Meyn stability criterion, pertain to the nature of the application problem and will be shown to hold through domain-specific analysis---though the need to ensure these assumptions indirectly influences the algorithm's formulation. Additional abstract conditions, such as those on noise terms, are motivated by applications and will either be verified directly or translated into implementable algorithmic requirements.

The following remarks clarify the roles of these conditions in the SA framework, framed in the RL context to prepare for their use in developing and analyzing RVI Q-learning. Readers primarily interested in the RL algorithm may choose to revisit this material later for a more integrated view of the analysis and algorithm design.

To apply the asynchronous SA results above, we will cast RVI Q-learning in the form of \eqref{eq-alg0} and prove that Assums.~\ref{cond-h} and \ref{cond-ns}, concerning the function $h$ and the noise terms, along with the conditions on the associated ODE and the set $E_h$, hold based on the formulation of RVI Q-learning and the structural properties of weakly communicating SMDPs. 

Regarding the noise conditions in Assum.~\ref{cond-ns}:
\begin{itemize}[leftmargin=0.4cm,labelwidth=!]
\item The biased noise terms $\{\epsilon_{n}\}$ arise in SMDPs due to the dependence of the function $h$ on expected holding times (i.e., expected durations between state transitions), which are estimated from data with increasing accuracy by the RL algorithm.
\item Assumption~\ref{cond-ns}(i) relaxes the standard condition on the martingale difference noise terms $\{M_n\}$, which uses a deterministic constant $K$ instead of the $K_n$'s in the conditional variance bounds. For SMDPs, the standard condition on $\{M_n\}$ suffices when a known lower bound on expected holding times is available; otherwise, the more general condition Assum.~\ref{cond-ns}(i) is needed.
\item For MDPs, which are SMDPs with unit holding times, all $\epsilon_n = 0$, so the standard condition on $\{M_n\}$ suffices.
\end{itemize} 
See Lem.~\ref{lem-cond-noise} in Sec.~\ref{sec-4.1.1} for exactly how Assum.~\ref{cond-ns} is used in our analysis of RVI Q-learning for SMDPs.

Regarding the stepsize and asynchrony conditions in Assums.~\ref{cond-ss} and~\ref{cond-us}, together they establish a \emph{partial asynchrony} mechanism that aligns the algorithm’s asymptotic behavior, on average, with that of a synchronous counterpart (cf.\ \cite{Bor98} and \cite[Rem.~2.3]{YWS25a}). This partial asynchrony is crucial for RVI Q-learning in average-reward problems, where the underlying mappings generally lack contraction or nonexpansion properties (unlike Q-learning for discounted-reward or some total-reward problems, where stepsizes and asynchronous update schedules can be chosen more flexibly \cite{Tsi94,YuB13}). As noted earlier, these conditions align with those used in the original RVI Q-learning algorithm \cite{ABB01}, and they will likewise be imposed on our algorithm.

Class-1 and class-2 stepsizes both satisfy Assum.~\ref{cond-ss}, as does the larger stepsize choice $\alpha_n = \tfrac{\ln n}{n}$. To ensure Assum.~\ref{cond-us} on asynchrony when using these stepsizes, Assum.~\ref{cond-us}(i) is straightforward to handle, as it simply requires every component to be updated sufficiently frequently. The additional requirement in Assum.~\ref{cond-us}(ii) is essentially about balanced updates across all components, with the balance measured in cumulative stepsizes (corresponding to the local algorithmic time of each component) rather than update counts. Since this condition involves both the stochastic update schedules and the random stepsizes used in the updates, it may be less obvious how to satisfy it. We discuss this requirement in more detail below through several examples.

For Class-1 stepsizes, Assum.~\ref{cond-us}(ii) is satisfied in the simple setting where, for every component $i \in \I$, the update frequencies $\tfrac{\nu(n,i)}{n}$ converge a.s.\ (in addition to being bounded away from zero as required by Assum.~\ref{cond-us}(i)); see \cite[Ex.\ 3]{WYS24} for the verification. One way to ensure this convergence is to select components for updating $x_n$ so that the selection sequence eventually follows an irreducible Markov chain on $\I$. The additional asynchrony condition Assum.~\ref{cond-mus} then also holds for any $0 < \gamma < \tfrac{1}{2}$, since the law of the iterated logarithm gives a sufficiently fast rate of convergence for the update frequencies, as noted in \cite[Rem.~2.6(b)]{YWS25a}.

For the larger stepsize choice $\alpha_n = \tfrac{\ln n}{n}$, the component selection scheme just mentioned also ensures that Assum.~\ref{cond-us}(ii) holds. On the other hand, for the smaller, Class-2 stepsizes, Assum.~\ref{cond-us}(ii) is always satisfied under Assum.~\ref{cond-us}(i). We verify the first case below; the second is simpler and follows by similar calculations.

\begin{example}[{\normalfont Verifying the asynchrony condition for larger stepsizes}] \label{ex-stepsize-asynchrony} \rm \hfill \\
Let $\alpha_n =  \tfrac{\ln n}{n}$. For $x > 0$, using the definition of $N(n,x)$ and an integral approximation, we obtain
\begin{equation} \label{eq-ex-stepsize-1}
  x = \lim_{n \to \infty} \sum_{k=n}^{N(n,x)} \alpha_k = \lim_{n \to \infty} \tfrac{1}{2} \left( \big( \ln N(n,x) \big)^2 -  ( \ln n)^2 \right). 
\end{equation} 
Since $\ln N(n,x) +  \ln n \to \infty$, it follows that
\begin{equation} \label{eq-ex-stepsize-2}
    \ln N(n,x) -  \ln n \to 0, \ \ \ \text{as} \ n \to \infty. 
\end{equation}

Suppose that, a.s., for each $i \in \I$, $p_{n,i} : = \tfrac{\nu(n,i)}{n}$ converges to some $p_i > 0$ (possibly sample-path-dependent) with
\begin{equation} \label{eq-ex-stepsize-3}
   \ln n \, |p_{n,i} - p_i| \to 0,  \ \ \ \text{as} \ n \to \infty.
\end{equation}
This rate condition is mild. It holds, for instance, when the component selection sequence eventually follows an irreducible finite-state Markov chain, by the law of the iterated logarithm. It can also be ensured by deliberate design of update schedules. Under this rate condition, we verify that 
\begin{equation} \label{eq-ex-stepsize-4}
  \sum_{k = \nu(n,i)}^{\nu(N(n,x), i)} \alpha_k \to x \ \  a.s., \quad \text{as} \ n \to \infty,
\end{equation}  
which implies Assum.~\ref{cond-us}(ii).

Indeed, writing $N = N(n,x)$ to simplify notation and using an integral approximation, we have 
$$ \sum_{k = \nu(n,i)}^{\nu(N, i)} \alpha_k = \tfrac{1}{2}  \big( \ln \nu(N, i) \big)^2 -  \tfrac{1}{2} \big( \ln \nu(n,i) \big)^2 + o(1),$$
where $o(1)$ denotes a term that vanishes as $n \to \infty$. Since $\nu(N, i) =  p_{N, i} N$ and $\nu(n, i) =  p_{n, i} n$,
\begin{align*}
 \big( \ln \nu(N, i) \big)^2 & =  ( \ln N )^2 + 2 \ln p_{N, i} \ln N  +  (\ln p_{N, i})^2, \\
 \big( \ln \nu(n,i) \big)^2 & = ( \ln n )^2 + 2 \ln p_{n, i} \ln n  +  (\ln p_{n, i})^2.
\end{align*} 
The difference between these two expressions equals $( \ln N )^2 - ( \ln n )^2 + o(1)$. To see this, note that $\ln p_{N, i}$ and $\ln p_{n, i}$ both converge to $\ln p_i$ as $n \to \infty$, and by \eqref{eq-ex-stepsize-2} and \eqref{eq-ex-stepsize-3},
$$  \ln p_{N, i} \ln N -  \ln p_{n, i} \ln n = \ln p_i (\ln N - \ln n) +  ( \ln p_{N, i} - \ln p_i) \ln N - (\ln p_{n, i} - \ln p_i) \ln n   \to 0.$$
The desired convergence in \eqref{eq-ex-stepsize-4} then follows from \eqref{eq-ex-stepsize-1}. 

We remark, however, that the stepsize choice $\alpha_n =  \tfrac{\ln n}{n}$ is too large for the shadowing analysis underlying the sharpened Thm.~\ref{thm-3} and is therefore not covered by that theorem. \myqed
\end{example}

As to the remaining conditions involved in Thm.~\ref{thm-3}, they can also be satisfied through appropriate choices of stepsizes and asynchronous update schedules for RVI Q-learning. In particular, for SMDPs, since the biased noise terms $\{\epsilon_n\}$ arise from estimated expected holding times, the additional noise condition Assum.~\ref{cond-mns} with $\mu_\delta < - L_h$ can be translated into algorithmic requirements on the estimation of expected holding times and ensured accordingly. Moreover, an effective upper bound on $L_h$ can be obtained with minimal model knowledge---specifically, a lower bound on the minimum expected holding time---and used in place of the exact value of $L_h$ when setting the threshold for the stepsize scaling parameter $A$. (See Assum.~\ref{cond-rvi-alg2}, Ex.~\ref{ex-beta}, Rem.~\ref{rem-mdp}, Thm.~\ref{thm-rvi-ql2}, and its proof in Sec.~\ref{sec-shad2}.) 

In the special case of MDPs, where all $\epsilon_n = 0$, Assum.~\ref{cond-mns} holds automatically with $\mu_\delta = - \infty$, and the minimum expected holding time is $1$---no model knowledge is required.

\section{RVI Q-Learning in Average-Reward SMDPs} \label{sec-rvi}

We begin with an overview of average-reward SMDPs, their optimality properties, and Schweitzer's classical RVI algorithm (Sec.~\ref{sec-3.1}), before introducing our generalized RVI Q-learning algorithm (Sec.~\ref{sec-3.2}) and presenting its convergence properties (Sec.~\ref{sec-3.3}).

\subsection{Average-Reward Weakly Communicating SMDPs} \label{sec-3.1}

We consider a standard finite state and action SMDP with state space $\S$ and action space $\A$. In this framework, the system's evolution and the decision-maker's actions follow specific rules (see, e.g., \cite[Chap.\ 11]{Put94}). When the system is in state $s \in \S$ and action $a \in \A$ is taken, the system transitions to state $S$ at some random time $\tau \geq 0$, known as the \emph{holding time}, and incurs a random reward $R$ upon this transition. The joint probability distribution of $(S, \tau, R)$ for this transition from $(s,a)$ is given by a (Borel) probability measure $\mathbb{P}_{sa}$ on $\S \times \R_+ \times \R$.
Let $\E_{sa}$ denote the expectation operator with respect to (w.r.t.)\ $\mathbb{P}_{sa}$. 

\begin{assumption}[Conditions on the SMDP model] \label{cond-smdp} \hfill
\begin{itemize}[leftmargin=0.7cm,labelwidth=!] 
\item[\rm (i)] For some $\epsilon > 0$, $\mathbb{P}_{sa}( \tau \leq \epsilon) < 1$ for all $s \in \S$ and $a \in \A$.
\item[\rm (ii)] For all $s \in \S$ and $a \in \A$, $\E_{sa}[ \tau^2] < \infty$ and $\E_{sa}[ R^2] < \infty$.
\end{itemize}
\end{assumption} 

Assumption~\ref{cond-smdp}(i) is standard and prevents an infinite number of state transitions in a finite time interval \cite[Lem.\ 1]{yushkevich1982semi}. Assumption~\ref{cond-smdp}(ii) is imposed for our RL convergence analysis, and we will discuss later how some of our results extend under weaker conditions (see Rem.~\ref{rem-relax-variance-cond}). For the optimality results reviewed below, it suffices that the expected holding times and expected rewards incurred with each state transition are finite. 

Regarding the rules for decision-making, actions are applied initially at time $0$ and subsequently at discrete moments upon state transitions. For notational simplicity, we assume (without loss of generality) that all actions from $\A$ are admissible at each state. At the $n$th transition moment, before selecting the next action, the decision-maker knows the \emph{history} of states, actions, rewards, and holding times realized up to that point, denoted by 
$h_n \= (s_0, a_0, r_1, \tau_1; \, \ldots; \, s_{n-1}, a_{n-1}, r_n, \tau_n; s_n).$ 
The decision-maker employs a randomized or nonrandomized decision rule, represented by a Borel-measurable stochastic kernel $\pi_n: h_n \mapsto \pi_n(\,\cdot\, | h_n) \in \P(\A)$, to select the next action $a_n$ based on the history $h_n$. The collection $\pi\=\{\pi_n\}_{n \geq 0}$ of these decision rules is called a \emph{policy}. The set of all such policies is denoted by $\Pi$. 

Our primary focus will be on \emph{stationary policies}, which select the actions $a_n$ based solely on the states $s_n$ in a time-invariant manner. Such a policy can be represented by a mapping $\pi: \S \to \P(\A)$ or $\pi: \S \to \A$, depending on whether the employed decision rule is randomized or nonrandomized.

\begin{remark} \rm
In the standard SMDP formulation adopted above, control actions are applied only at the initial time and at state transition epochs. This modeling restriction is most natural for continuous-time decision systems with event-triggered actions or exponentially distributed holding times. For optimal control in more complex continuous-time problems, actions may need to change continuously between state transitions (see, e.g., \cite{Yus80a,Yus80b,Can84}). Nonetheless, the standard SMDP model remains useful for numerical solutions to such problems through the Markov chain approximation method \cite{KuD01} combined with policy approximation techniques, such as piecewise constant policies and quantized actions (see \cite{PrY24} and the references therein). \myqed
\end{remark}

We use an average-reward criterion to evaluate policy performance. The \emph{average reward rate} of a policy $\pi$ is defined for each initial state $s \in \S$ as: 
\begin{equation} \label{eq: r-pi smdp}
 r(\pi, s) \= \, \textstyle{ \liminf_{t \to \infty}  t^{-1} \, \E^\pi_s \left [\sum_{n = 1}^{N_t} R_n \right]\!,}
\end{equation} 
where $\E^\pi_s [ \,\cdot ]$ denotes the expectation w.r.t.\ the probability distribution of the random process $\{(S_n, A_n, R_{n+1}, \tau_{n+1})\}_{n \geq 0}$ induced by the policy $\pi$ and initial state $S_0=s$. The summation $\sum_{n = 1}^{N_t} R_n$ represents the total rewards received by time $t$, where $N_t$ counts the number of transitions by that time, defined as $N_t = \max \{ n \mid t_n  \leq t\}$ with $t_n \= \sum_{i=1}^n \tau_i$ and $t_0=0$. Assumption \ref{cond-smdp} ensures that $r(\pi, s)$ is well-defined, real-valued, and uniformly bounded across policies $\pi$ and states $s$ (as can be shown based on \cite[Lem.\ 1]{yushkevich1982semi}). If the policy $\pi$ is stationary, the $\liminf$ in definition \eqref{eq: r-pi smdp} can be replaced by $\lim$ according to renewal theory (see \citep{Ros70}). A policy is called \emph{optimal} if it achieves the \emph{optimal reward rate} $r^*(s) \= \sup_{\pi \in \Pi} r(\pi, s)$ for \emph{all} initial states $s \in \S$. 

Based on \cite[Thms.\ 2 and 3]{yushkevich1982semi}, under Assum.~\ref{cond-smdp}, there exists a nonrandomized stationary optimal policy, and such a policy, along with $r^*$, can be identified from a solution to the AOE (\emph{average-reward optimality equation}). This equation can be expressed in terms of either state values or state-and-action values. We opt for the latter form as it aligns with the RL application under consideration. 
Let 
$$r_{sa} \= \E_{sa} [ R ], \qquad t_{sa} \= \E_{sa} [ \tau], \qquad p_{ss'}^a \=  \mathbb{P}_{sa} ( S = s')$$ 
denote, respectively, the expected reward, the expected holding time, and the probability of transitioning to state $s'$ from state $s$ with action $a$. 
We focus on the case where $r^*$ is independent of the initial state. In this case, we seek a solution $(\bar r, q) \in \R \times \R^{|\S \times \A|}$ to the AOE: 
\begin{equation}
q(s, a)  =  r_{sa} - t_{sa} \cdot \bar r  +  \sum_{s' \in \S} p_{ss'}^a \max_{a' \in \A} q(s', a'),\qquad \ \, \forall \, s \in \S, \, a \in \A.  \label{eq-opt}
\end{equation} 
Recall that, in this case, solutions to~\eqref{eq-opt} exist, and their $\bar r$-components always equal the constant optimal reward rate $r^*$. Moreover, any stationary policy that solves the corresponding maximization problems in the right-hand side (r.h.s.)\ of AOE is optimal  \cite{ScF78,yushkevich1982semi}. Bounds on $r^*$ and performance bounds on policies can also be derived from AOE~\cite[Thm.\ 1]{platzman1977improved}. 

For RL, we shall focus on \emph{weakly communicating SMDPs}, wherein $r^*$ remains constant. These SMDPs are defined by their state communication structure \cite{Bat73,platzman1977improved,Put94}: they possess a unique \emph{closed communicating class}---a set of states such that starting from any state in the set, every state in it is reachable with positive probability under some policy, but no states outside it are ever visited under any policy. The remaining states, if any, are transient under all policies. 

Solutions to AOEs in weakly communicating SMDPs exhibit two structural properties that are important to the RL context we consider later. 
\begin{itemize}[leftmargin=0.4cm,labelwidth=!]
\item First, unlike a unichain SMDP, in a weakly communicating SMDP, the solutions for $q$ in \eqref{eq-opt} may not be unique up to an additive constant. Instead, they can have multiple degrees of freedom, depending on the recurrence structures of the Markov chains $\{S_n\}$ induced by stationary optimal policies, as characterized by Schweitzer and Federgruen \cite{ScF78} (see \cite[Sec.\ 2.2]{WYS24} for some illustrative examples). 
\item Second, despite this lack of uniqueness, these solutions can only `escape to $\infty$' asymptotically along the directions represented by constant vectors (i.e., multiples of the all-ones vector). 
\end{itemize}
This second property is encapsulated in the fact that in weakly communicating SMDPs with zero rewards, AOEs have unique solutions (up to an additive constant) represented by constant vectors. This fact is particularly important for the stability of our RL algorithms, and we state it in the lemma below. It can be inferred from the general theory of \citet[Thms.~3.2 and~5.1]{ScF78}, or proved directly (see \cite[Lem.\ 5.1 and Rem.\ 7.1(b)]{WYS24}):

\begin{lemma} \label{lem: sol-struc} 
Suppose Assum.~\ref{cond-smdp} holds and the SMDP is weakly communicating. 
If all rewards $\{r_{sa} \}_{s \in \S, a \in \A}$ are zero, the only solutions to AOE \eqref{eq-opt} are $\bar r = 0$ with $q(\cdot) \equiv c, c \in \R$.
\end{lemma}

When $r^*$ is constant, Schweitzer's RVI algorithm \citep{Sch71} can be applied to solve AOE. We describe a version of this algorithm, which can be compared with the Q-learning algorithms introduced later. Given an initial vector $Q_0 \in \R^{|\S \times \A |}$, compute $Q_{n+1} \in \R^{|\S \times \A |}$ iteratively for $n \geq 0$ as follows: for all $(s,a) \in \S \times \A$,
\begin{equation} \label{eq-s-rvi}
  Q_{n+1}(s,a) = Q_n(s,a) + \bar \alpha \Big( \frac{r_{sa} + \sum_{s' \in \S} p_{ss'}^a \max_{a' \in \A} Q_{n}(s', a')  - Q_n(s,a)}{t_{sa}} -  f(Q_n) \Big),
\end{equation}   
where, for a fixed state and action pair $(\bar s, \bar a)$,
$$f(Q_n) \=  t_{\bar s \bar a}^{-1}\Big(r_{\bar{s}\bar{a}} + \sum_{s' \in \S} p_{\bar ss'}^{\bar a} \max_{a' \in \A} Q_{n}(s', a')  - Q_n(\bar s, \bar a)\Big),$$ 
and the stepsize $\bar \alpha$ can be chosen within $(0, \min_{s \in \S, a \in \A} t_{sa})$. This algorithm converges when $r^*$ remains constant; specifically, as $n \to \infty$, $f(Q_n) \to r^*$ and $Q_n \to \bar q$, a solution of AOE \eqref{eq-opt} \citep{platzman1977improved,Sch71,ScF77}.

\begin{remark} \rm
The RVI Q-learning algorithm introduced next is a model-free, asynchronous stochastic counterpart to Schweitzer's RVI algorithm. To focus the discussion, we assume it operates in weakly communicating SMDPs under the preceding model conditions and average-reward criterion. However, our results presented below apply more broadly to scenarios where AOE \eqref{eq-opt} holds and the SMDP with zero rewards exhibits the solution structure asserted in Lem.~\ref{lem: sol-struc}. 
In particular:\\
\noindent (a) Our results extend to an alternative average-reward criterion (criterion $w_2$ in \cite[Eq.\ (4)]{yushkevich1982semi}), which is suitable for incremental reward accrual rather than lump-sum rewards per transition. AOE holds under this criterion for fairly general continuous reward generation mechanisms \cite[Thm.\ 3]{yushkevich1982semi}.\\
\noindent (b) Beyond weakly communicating SMDPs, our results apply broadly to SMDPs with constant $r^*$, where a stationary policy that applies every action with positive probability induces a Markov chain $\{S_n\}$ with a single recurrent class (possibly with transient states). Based on \cite{ScF78}, these SMDPs exhibit the required solution structure described in Lem.~\ref{lem: sol-struc}. However, when the SMDP is not weakly communicating, verifying that $r^*$ is constant may require problem-specific knowledge. \myqed
\end{remark}
\vspace*{-0.1cm} 

\subsection{RVI Q-Learning: Generalized Formulation} \label{sec-3.2}

The RVI Q-learning algorithm we introduce below is indeed a family of average-reward RL algorithms based on the RVI approach. These algorithms operate without knowledge of the SMDP model, using random transition data from the SMDP to solve AOE \eqref{eq-opt}. Unlike the classical RVI algorithm \eqref{eq-s-rvi}, these algorithms are asynchronous, updating only for a subset of state-action pairs at each iteration based on available data. The critical scaling by the expected holding times $t_{sa}$, essential for the convergence of \eqref{eq-s-rvi} in SMDPs, is applied here using data-based estimates, such as sample averages.

Another key difference from the classical RVI algorithm is the choice of function $f$ for estimating the optimal reward rate. Due to stochasticity and asynchrony, as first proposed in \citet{ABB01}, the choice of $f$ must provide the learning algorithms with a `self-regulating' mechanism to ensure stability and convergence.

An immediate predecessor of the following algorithm for SMDPs was introduced by Wan et al.\ \cite{WNS21b} in the context of hierarchical control for average-reward MDPs. The algorithmic framework presented here builds on the original formulation of RVI Q-learning for MDPs by Abounadi et al.\ \cite{ABB01} and its recent extensions by Wan et al.\ \cite{WNS21a,WNS21b}. Our major generalization introduced below is in the class of functions $f$ used.

Consider a weakly communicating SMDP under Assum.~\ref{cond-smdp}. Let $d = |\S \times \A|$. The RVI Q-learning algorithm maintains estimates of state-action values and expected holding times, represented by $d$-dimensional vectors $Q_n$ and $T_n \geq 0$ at each iteration $n$. The initial $Q_0$ and $T_0$, considered as given, can be chosen arbitrarily. The algorithm uses two separate sequences of deterministic, diminishing stepsizes, $\{\alpha_k\}$ and $\{\beta_k\}$, for updating $Q_n$ and $T_n$, respectively, along with a given sequence of diminishing positive scalars, $\eta_n$, to lower-bound the estimated holding times at each iteration. (If a positive lower bound on $\min_{s \in \S, a \in \A} t_{sa}$ is known a priori, it can replace the $\eta_n$'s; for example, $1$ in the case of hierarchical control in MDPs \cite{WNS21b}.) The estimates $Q_n$ and $T_n$ are updated iteratively as follows. At iteration $n \geq 0$: 
\begin{itemize}[leftmargin=0.5cm,labelwidth=!]
\item A subset $Y_n \not=\varnothing$ of state-action pairs is randomly selected. For each pair $(s,a) \in Y_n$, there is a freshly generated data point, consisting of a random state transition, holding time, and reward  $(S_{n+1}^{sa}, \tau_{n+1}^{sa}, R_{n+1}^{sa})$ jointly distributed according to $\mathbb{P}_{sa}$. 
\item Use these transition data to update the corresponding components of $Q_n$ and $T_n$:
\begin{itemize}[leftmargin=0.15cm,labelwidth=!]
\item[] for $(s, a) \not\in Y_n$:  $Q_{n+1}(s, a) \= Q_{n}(s, a)$ and $T_{n+1}(s, a) \= T_{n}(s, a)$;
\item[] for $(s, a) \in Y_n$:\vspace*{-2pt}
\begin{align}
   Q_{n+1}(s, a)  & \= Q_{n}(s, a)   
    \nonumber \\
  &  +  \alpha_{\nu(n, (s, a))}  \left( \frac{R_{n+1}^{sa}  + \max_{a' \in \A} Q_n (S_{n+1}^{sa}, a') - Q_n(s, a)}{T_n(s, a) \vee \eta_n} -  f(Q_n) \right),  \label{eq: q-ite}  
\end{align}
\begin{equation}
    T_{n+1}(s, a)  \= T_{n}(s, a) + \beta_{\nu(n, (s,a))} (\tau_{n+1}^{sa} - T_{n}(s, a)).  
    \label{eq: t-ite}
\end{equation}
\end{itemize}
\end{itemize}
In the above, $f : \R^d \to \R$ is a Lipschitz continuous function with additional properties to be given shortly. The term $\nu(n,(s,a))  \= \sum_{k=0}^{n-1} \ind \{ (s,a) \in Y_k \}$ is the cumulative count of how many times the state-action pair $(s,a)$ has been chosen prior to iteration $n$. Stochastic gradient descent is applied in \eqref{eq: t-ite} to estimate the expected holding time $t_{sa}$, with a standard stepsize sequence $\beta_k \in [0,1]$, $k \geq 0$. 

The algorithmic conditions are summarized below. See Sec.~\ref{sec-sa-cond} for a discussion on how the asynchrony conditions required by Assum.~\ref{cond-us} can, in particular, be ensured.

\begin{assumption}[Algorithmic requirements]\label{cond-rvi-alg}\hfill
 \begin{itemize}[leftmargin=0.7cm,labelwidth=!]
\item[\rm (i)] The stepsizes $\{\alpha_n\}$ satisfy Assum.~\ref{cond-ss}. The asynchronous update schedules are such that $\{\nu(n, \cdot)\}$ satisfies Assum.~\ref{cond-us} with the space $\I = \S \times \A$.
\item[\rm (ii)] The stepsizes $\{\beta_n\}$ satisfy $ \beta_n \in [0, 1]$ for $n \geq 0$, $\sum_{n} \beta_n = \infty$, and $\sum_{n} \beta_n^2 < \infty$.
\item[\rm (iii)] The sequence $\{\eta_n\}$ satisfies $\eta_n > 0$ for all $n \geq 0$ and $\lim_{n \to \infty} \eta_n = 0$.
\end{itemize}    
\end{assumption}

Anticipating the application of the Borkar--Meyn stability criterion, and with the solution structure of a weakly communicating SMDP in mind (Lem.~\ref{lem: sol-struc}), we now introduce our conditions on the function $f$:

\begin{definition} \rm \label{def-sistr} 
We call $g: \R^d \to \R$ \emph{strictly increasing under scalar translation} (SISTr) if, for every $x \in \R^d$, the function $c \in \R \mapsto g(x + c)$ is strictly increasing and maps $\R$ onto $\R$. If this condition holds at a specific point $x$, we say $g$ is \emph{SISTr at} $x$.
\end{definition}

\begin{assumption}[Conditions on function $f$]\label{cond-f} \hfill
\begin{itemize}[leftmargin=0.7cm,labelwidth=!]
\item[\rm (i)] $f$ is Lipschitz continuous and SISTr.
\item[\rm (ii)] As $c \uparrow \infty$, the function $f(c \, \cdot)/c \overset{p}{\to} f_\infty : \R^d \to \R$, and $f_\infty$ is SISTr at the origin.
\end{itemize}
\end{assumption}

It is possible to relax the SISTr condition on $f$ in this assumption; see Rem.~\ref{rmk-rel-sistr} after our convergence proof. 

The existence of the scaling limit $f_\infty$ implies that it is Lipschitz continuous and positively homogeneous (i.e., $f_\infty(c x) = c f_\infty(x)$ for $c \geq 0$), with $f_\infty(0) = 0$. Moreover, $f(c \, \cdot)/c \overset{u.c.}{\to} f_\infty$ as $c \uparrow \infty$. Since $f$ is SISTr, $f_\infty$ is nondecreasing under scalar translation at each point, though it only needs to be SISTr at the origin, which is strictly weaker than requiring $f_\infty$ to be SISTr at all points, as the following example demonstrates.

\begin{example} \label{cntex-f} \rm
This example shows a function $f : \R^2 \to \R$ that satisfies Assum.~\ref{cond-f} without its scaling limit $f_\infty$ being SISTr at all points. Express each $x \in \R^2$ as $x = x_a v_a + x_c v_c$ w.r.t.\ the basis $v_a = (1, -1)$ and $v_c = (1, 1)$. Divide $\R^2$ into three regions and define $f$ on each region as follows:
\begin{equation} 
    f(x) \= \begin{cases}
        2 x_c \phi(x_a) & \text{if} \ x_a \geq 0, \ 0 \leq x_c \leq \frac{x_a}{2}; \\
        2 (x_a - x_c) \phi(x_a) + (2 x_c - x_a)& \text{if} \ x_a \geq 0, \ \frac{x_a}{2} < x_c \leq  x_a; \\
        x_c & \text{otherwise},
        \end{cases}
\end{equation}
where $\phi(x_a) \= 1 - \frac{e^{-x_a}}{2}$. The scaling limit $f_\infty$ is then given by
\begin{equation} 
    f_\infty(x) \= \begin{cases}
 2 x_c  & \text{if} \ x_a \geq 0, \ 0 \leq x_c \leq \frac{x_a}{2}; \\
        x_a & \text{if} \ x_a \geq 0, \ \frac{x_a}{2} < x_c \leq  x_a; \\
        x_c & \text{otherwise}.
        \end{cases}
\end{equation}
It is straightforward to verify that $f$ and $f_\infty$ are Lipschitz continuous, $f$ is SISTr, and $f_\infty$ is SISTr at the origin. Specifically, taking $x^o$ as the origin, we have $f(x^o + c) = c$ for $c \in \R$. 
However, at points $\bar x = \bar a v_a$ with $\bar a > 0$, we have $f_\infty(\bar x + c) = \bar a$ for all $c \in [\frac{\bar a}{2}, \bar a]$, so $f_\infty$ is not SISTr at such points $\bar x$. \myqed
\end{example}

Assumption~\ref{cond-f} significantly expands the scope of the previous assumptions on $f$ introduced in \cite{ABB01,WNS21a} for RVI Q-learning. Let us discuss this with some examples.

\begin{example}[Examples of $f$]  \rm
\label{ex-f}
In addition to Lipschitz continuity, the function $f$ considered in \cite{ABB01,WNS21a} satisfies the following two conditions: 
\begin{itemize}[leftmargin=0.6cm,labelwidth=!]
\item For some scalar $u > 0$, $f(x + c) = f(x) + c u$ for all $x \in \R^d$ and $c \in \R$; and 
\item for $c \geq 0$, $f(c x) = f(0) + c (f(x) - f(0))$.
\end{itemize}
(The case $u = 1$ is equivalent to the conditions originally introduced by \cite{ABB01}, and the generalization to $u > 0$ is due to \cite{WNS21a}.) Such functions $f$ satisfy Assum.~\ref{cond-f} with scaling limit $f_\infty(x) = f(x) - f(0)$. Examples include affine functions of the form
$f(x) = b + \theta^\top x$, where $b \in \R$, $\theta \in \R^d$, and $\sum_{i=1}^d \theta_i > 0$, as well as nonlinear functions such as $f(x)  = b + \beta \max_{i \in D} x_i$ or $f(x)  = b + \beta \min_{i \in D} x_i$, where $b \in \R$, $\beta > 0$, and $D \subset \{1, 2, \ldots, d\}$.
 
For functions that satisfy our Assum.~\ref{cond-f} but do not necessarily meet the conditions of \cite{ABB01,WNS21a}, consider examples such as $f(x) = \max \{ g_1(x), \ldots, g_m(x)\}$ or $\min \{ g_1(x)$, $\ldots, g_m(x)\}$, where each function $g_k$ fits into one of the aforementioned types.

In general, if $g_1, \ldots, g_m$ satisfy Assum.~\ref{cond-f}, then $f(x) = \psi(g_1(x), \ldots, g_m(x))$ also satisfies Assum.~\ref{cond-f}, where $\psi : \R^m \to \R$ possesses the following properties:
\begin{itemize}[leftmargin=0.75cm,labelwidth=!]
\item[(i)] Lipschitz continuity and strict monotonicity, that is, $\psi(y) > \psi(y')$ if $y > y'$ component-wise.
\item[(ii)] $\psi(y) \to \infty$ as $\min_{i \leq m} y_i \to \infty$, and $\psi(y) \to - \infty$ as $\max_{i \leq m} y_i \to -\infty$.
\item[(iii)] As $c \uparrow \infty$, $\psi (c \,\cdot)/c \overset{p}{\to} \psi_\infty : \R^m \to \R$, and $\psi_\infty$ possesses properties (i,\,ii).
\end{itemize}
In particular, properties (i--ii) ensure that $f$ is SISTr, while property (iii) ensures that $f_\infty$ is SISTr at the origin, with 
$$f_\infty(x) = \psi_\infty\!\big(g_{1, \infty}(x), \, \ldots,\, g_{m, \infty}(x) \big).$$ 

Beyond weighted averages, medians, and max/min functions, there is a large class of functions with these properties. As a simple example, suppose that, based on problem-specific knowledge, we know that $r^* \in (a, b)$. Let $m = 1$ and $g(x) = \max_{i \leq d} x_i$. Define the
piecewise linear function $\psi:\R\to\R$ by 
\begin{equation}
  \psi(y) = \begin{cases} 
     y, &  y \in [a, b], \\
     \beta (y - a) + a, &  y < a, \\
      \beta (y - b) + b, &   y > b,
\end{cases} \label{eq-ex-psi}      
\end{equation}
where $\beta > 1$. Then the function
\begin{equation} \label{eq-ex-f}
   f(x) : = \psi (g(x)), \quad \text{with scaling limit} \ f_\infty(x) = \beta \max_{i \leq d} x_i,
\end{equation}
satisfies Assum.~\ref{cond-f}. Using this function $f$ with a large value of $\beta$, instead of $g$, in RVI Q-learning, allows the algorithm to make stronger corrections to its updates whenever the reward rate estimate $g(Q)$ lies outside the known range $(a,b)$ of $r^*$. We will revisit this example later to discuss another useful property of this choice of $f$ (see Rem.~\ref{rem-mdp}).

More generally, composition with $\psi$ allows for the transformation or combination of different estimates, offering greater flexibility in estimating the optimal reward rate than earlier formulations of RVI Q-learning. \myqed
\end{example}

The next lemma gives an implication of Assum.~\ref{cond-f}(i). We will apply it below to characterize the set of solutions to AOE that are constrained by $f(q) = r^*$, which is the target set for RVI Q-learning to converge to. Another implication of Assum.~\ref{cond-f}(i) will be given later in Lem.~\ref{lem-f2} during our convergence analysis, where the monotonicity property of $f$ will be critical for the `self-regulating' behavior of RVI Q-learning.

\begin{lemma} \label{lem-f}
Let $f$ satisfy Assum.~\ref{cond-f}(i), and let $\ell \in \R$. 
For each $x \in \R^d$, there exists a unique $c_x \in \R$ such that $f(x + c_x) = \ell$, and the function $x \mapsto c_x $ is continuous. 
\end{lemma}

\begin{proof}
Since the function $c \mapsto f(x + c)$ maps $\R$ one-to-one onto $\R$ under Assum.~\ref{cond-f}(i), $c_x$ exists and is unique. To show the continuity of the function $x \mapsto c_x$, suppose, for contradiction, that it is discontinuous. Then there exist some $\bar x \in \R^d$, $\delta > 0$, and a sequence $x_n \to \bar x$ such that $|c_{x_n} - c_{\bar x}| > \delta$ for all $n$. Given $f(x_n + c_{x_n}) = \ell$ for all $n$, the Lipschitz continuity of $f$ and the uniqueness of $c_{\bar x}$ imply that $| c_{x_n} | \to \infty$ and $| f(\bar x + c_{x_n}) - f(x_n + c_{x_n})| \to 0$, hence $f(\bar x + c_{x_n}) \to \ell$. But, since $f$ is SISTr, $| c_{x_n} | \to \infty$ implies $|f(\bar x + c_{x_n})| \to \infty$. This contradiction proves that $x \mapsto c_x$ is continuous.
\end{proof}

Let $\Q$ be the set of solutions $q$ to AOE \eqref{eq-opt}. Let $\Q_f \= \{ q \in \Q \mid f(q) = r^* \}$. By SMDP theory (see Sec.~\ref{sec-3.1}), $\Q_f$ is the solution set of the equation
\begin{equation}
  q(s, a)  =  r_{sa} - t_{sa} \cdot f(q)  +  \sum_{s' \in \S} p_{ss'}^a \max_{a' \in \A} q(s', a'),\qquad \ \, \forall \, s \in \S, \, a \in \A.  \label{eq-for-rvi}
\end{equation} 
Consider also the case where all expected rewards $r_{sa} = 0$ and the function $f_\infty$ replaces $f$. Denote the corresponding set by $\Q^o_{f_\infty}$; that is, $\Q^o_{f_\infty}\!$ is the solution set of the equation
\begin{equation}
  q(s, a)  =   - t_{sa} \cdot f_\infty(q)  +  \sum_{s' \in \S} p_{ss'}^a \max_{a' \in \A} q(s', a'),\qquad \ \, \forall \, s \in \S, \, a \in \A.  \label{eq-for-rvi0}
\end{equation} 

\begin{prop} \label{prp-sol-set}
Consider an SMDP satisfying Assum.~\ref{cond-smdp}. 
If the SMDP is weakly communicating and the function $f$ satisfies Assum.~\ref{cond-f}, 
then $\Q_f$ is nonempty, compact, and connected, while $\Q^o_{f_\infty}\!$ contains only the origin in $\R^d$.
\end{prop}

\begin{proof} 
This proof generalizes our proofs in \cite[Thm.\ 5.1 and Lem.\ 6.1]{WYS24}, which reach the same conclusions under stronger conditions on $f$ from \cite{ABB01,WNS21a} (cf.\ Ex.~\ref{ex-f}).

Consider $\Q^o_{f_\infty}\!$ first. By Lem.~\ref{lem: sol-struc}, in a weakly communicating SMDP, a vector $q \in \R^d$ solves \eqref{eq-for-rvi0} if and only if $f_\infty(q) = 0$ and $q(\cdot) \equiv c$ for some $c \in \R$. The origin $0$ in $\R^d$ is the only solution because $f_\infty(0) = 0$ and $f_\infty$ is SISTr at $0$ by Assum.~\ref{cond-f}(ii). Thus, $\Q^o_{f_\infty}\! =\{0\}$.

Regarding $\Q_f$, note first that $\Q \not=\varnothing$ in a weakly communicating SMDP, and $\Q$ is connected by \cite[Thm.\ 4.2]{ScF78}. Lemma~\ref{lem-f} with $\ell = r^*$ implies that $\Q_f$ is the image of $\Q$ under the continuous mapping $q \mapsto q + c_q$. Consequently, $\Q_f$ is nonempty and connected. As the solution set of \eqref{eq-for-rvi}, the closedness of $\Q_f$ is obvious from the Lipschitz continuity of $f$. 

Finally, the boundedness of $\Q_f$ is established via proof by contradiction: If it is unbounded, let $\{q_n\}$ be a sequence in $\Q_f$ with $\|q_n\| \uparrow \infty$. Since $q_n$ solves \eqref{eq-for-rvi}, we divide both sides of this equation (with $q_n$ in place of $q$) by $\| q_n\|$, and then let $n \to \infty$. Noting that $f(c \,\cdot )/c \overset{u.c.}{\to} f_\infty$ as $c \uparrow \infty$ under Assum.~\ref{cond-f}, we obtain that any limit point $\bar q$ of $\{q_n / \| q_n\|\}$ solves \eqref{eq-for-rvi0}. But this is impossible since $\| \bar q \| = 1$, while $\Q^o_{f_\infty}\! = \{0\}$ as proved above. Consequently, $\Q_f$ must be bounded and hence compact. 
\end{proof}

\begin{remark} \label{rmk: Qf-dim} \rm
Under the same conditions as Prop.~\ref{prp-sol-set} and based on the theory from \cite{ScF78}, the solution set $\Q$ is homeomorphic to a convex polyhedron, with its dimension determined by the recurrence structures of the stationary optimal policies in the SMDP. Using this and Lem.~\ref{lem-f}, it can be shown that the set $\Q_f$ is homeomorphic to a convex polyhedron of exactly one fewer dimension. (This proof closely parallels our previous proof of \cite[Thm.\ 7.1]{WYS24}.) Thus, $\Q_f$ is generally not a singleton but rather a connected subset of infinitely many solutions to AOE~\eqref{eq-opt}. \myqed
\end{remark}

\subsection{Convergence Results} \label{sec-3.3}

We now present our convergence results for RVI Q-learning. The first theorem establishes convergence to the subset $\Q_f$ of AOE solutions, while the second addresses convergence to a single solution. Their proofs, given in the next section, relate RVI Q-learning to the general SA algorithm~\eqref{eq-alg0} and rely on Thms.~\ref{thm-ql} and~\ref{thm-3}.

\begin{theorem} \label{thm-rvi-ql}
Suppose Assum.~\ref{cond-smdp} holds and the SMDP is weakly communicating. Then, under Assums.~\ref{cond-rvi-alg} and~\ref{cond-f} for the algorithm \eqref{eq: q-ite}-\eqref{eq: t-ite}, almost surely:
\begin{itemize}[leftmargin=0.6cm,labelwidth=!]
\item[\rm (i)] $\{Q_n\}$ converges to a compact connected subset of $\Q_f$, with $f(Q_n) \to r^*$.
\item[\rm (ii)] The continuous trajectory $\bar x(\cdot)$ defined by $\{Q_n\}$ according to \eqref{eq-cont-traj2} (with $x_n = Q_n$) has the convergence property asserted in Thm.~\ref{thm-ql}(ii) with the set $E_h = \Q_f$.
\end{itemize}
\end{theorem}

\begin{remark} \rm
The a.s.\ convergence of $Q_n$ to the compact set $\Q_f$ by Thm.\ \ref{thm-rvi-ql}(i) and Prop.~\ref{prp-sol-set} implies that a.s.\ for all sufficiently large $n$, any stationary policy $\pi$ satisfying 
$$\{ a \in \A \mid \pi(a \,|\, s) > 0\} \subset \argmax_{a \in \A} Q_n(s, a), \quad \forall \, s \in \S$$ 
is optimal in the SMDP (see the proof of \cite[Thm.~3.2(ii)]{WYS24}). \myqed
\end{remark}

As noted earlier after Thm.~\ref{thm-ql}, Thm.\ \ref{thm-rvi-ql}(ii) shows that the iterates $\{Q_n\}$ may spend increasing amounts of ODE-time near each of infinitely many connected limit points in $\Q_f$ without converging to any single one. This behavior can resemble convergence, making it difficult in practice to determine from numerical data whether the algorithm is truly approaching a unique limit.
  
To ensure the convergence of $\{Q_n\}$ to a single point in $\Q_f$, we introduce additional conditions below, whose purpose is to enable the application of the shadowing-based SA convergence result, Thm.~\ref{thm-3}.

For a stepsize sequence $\{\alpha_n\}$, define the associated quantity
$$\ell(\{\alpha_n\}) \= \limsup_{n \to \infty} \tfrac{\ln \alpha_n}{\sum_{k=0}^{n} \alpha_k},$$ 
which characterizes how fast it decreases. For example, for $\alpha_n = \tfrac{1}{A n}$ and $\alpha_n = \tfrac{1}{A n \ln n}$, which are the class-1 and class-2 stepsizes defined in \eqref{def-stepsize-class} (Sec.~\ref{sec-2.1}), we have $\ell(\{\alpha_n\}) = - A$ and $-\infty$, respectively.

Let $L_f$ be the Lipschitz constant of $f$ w.r.t.\ $\|\cdot\|_\infty$. Define 
$$A_* \= \tfrac{2}{\tm} + L_f, \qquad \text{where} \ \ \tm \= \min_{(s,a) \in \S \times \A} t_{sa}.$$

\begin{assumption}[Additional algorithmic conditions]\label{cond-rvi-alg2}\hfill
 \begin{itemize}[leftmargin=0.7cm,labelwidth=!]
\item[\rm (i)] The stepsizes $\{\alpha_n\}$ belong to class 1 or 2, with scaling parameter $A > 0$.
\begin{itemize}[leftmargin=0.5cm,labelwidth=!]
\item[\rm a.] For class-1 stepsizes, $\tfrac{A}{2} > A_*$, and the asynchronous update schedules satisfy Assum.~\ref{cond-mus} with $\I = \S \times \A$ and $\gamma A > A_*$. 
\item[$\rm b.$] For class-2 stepsizes, $A > A_*$.
\end{itemize}
\item[\rm (ii)] The stepsizes $\{\beta_n\}$ satisfy, for some $\varsigma > A_*$, 
$$\beta_n \geq \varsigma \alpha_n \ \ \text{for all sufficiently large $n$}, \ \ \ \text{and} \ \  \tfrac{- \varsigma \ell(\{\beta_n\})}{2} > A_*.$$ 
This condition holds, for instance, if $\beta_n = \varsigma \alpha_n$ for all $n \geq 0$, where $\{\alpha_n\}$ satisfies the conditions in (i). 
\end{itemize}
\end{assumption}

\begin{remark} \rm
Recall from Sec.~\ref{sec-sa-cond} that if the selection of which components of $Q_n$ to update eventually follows an irreducible Markov chain on $\S \times \A$, then Assum.~\ref{cond-mus} holds for any $\gamma \in (0, \tfrac{1}{2})$ by the law of the iterated logarithm, and consequently Assum.~\ref{cond-rvi-alg2}(i.a) can be satisfied with $\tfrac{A}{2} > A_*$. \myqed
\end{remark}

\begin{example}[{\normalfont Admissible and non-admissible $\{\beta_n\}$ under Assum.~\ref{cond-rvi-alg2}(ii)}] \label{ex-beta} \rm \hfill \\
The natural choice $\beta_n = 1/n$, corresponding to the sample mean estimator of expected holding times, satisfies Assum.~\ref{cond-rvi-alg2}(ii) for both class-1 and class-2 stepsizes $\{\alpha_n\}$ under Assum.~\ref{cond-rvi-alg2}(i).

More generally, the relative rates of decrease of $\{\beta_n\}$ and $\{\alpha_n\}$ determine the
timescales of holding-time estimation and state-action value updates. Assum.~\ref{cond-rvi-alg2}(ii) allows holding-time estimation on the same or a faster timescale, but excludes a slower timescale, for which $\beta_n/\alpha_n \to 0$.

For example, when $\{\alpha_n\}$ is in class 2, the choice
$\beta_n = \frac{1}{B n (\ln n)^b}$, where $B > 0$ and $b \in [0,1)$,
satisfies Assum.~\ref{cond-rvi-alg2}(ii). Indeed, $\beta_n/\alpha_n \to \infty$, while $-\ell(\{\beta_n\}) = B$ for $b = 0$ and $-\ell(\{\beta_n\}) = \infty$ for $b \in (0,1)$. Thus, $\varsigma$ can be chosen sufficiently large to meet the requirements of the assumption.

On the other hand, the faster-timescale choice $\beta_n = n^{-b}$, with $b \in (1/2,1)$, has $\ell(\{\beta_n\}) = 0$ and therefore violates Assum.~\ref{cond-rvi-alg2}(ii), although it is allowed under Assum.~\ref{cond-rvi-alg}. \myqed
\end{example}

The next theorem strengthens Thm.~\ref{thm-rvi-ql} under the additional Assum.~\ref{cond-rvi-alg2}:

\begin{theorem} \label{thm-rvi-ql2}
Let the conditions in Thm.~\ref{thm-rvi-ql} hold; moreover, impose the stepsize and asynchony conditions from Assum.~\ref{cond-rvi-alg2}. Then the sequence $\{Q_n\}$ from algorithm \eqref{eq: q-ite}-\eqref{eq: t-ite} converges a.s.\ to a unique (sample-path-dependent) point in $\Q_f$.
\end{theorem}

As noted earlier, the conditions involved can be satisfied through appropriate choices of stepsizes and update schemes, without requiring model knowledge beyond a lower bound on $\tm$. When such a lower bound is not available \emph{a priori}, the theorem can still provide a qualitative convergence guarantee for $A$ and $\varsigma$ sufficiently large. Moreover, it can be turned into a high-probability guarantee if a lower bound on $\tm$ is estimated from data and incorporated into the algorithm for tuning its parameters.

We close this section with a discussion of two special cases---MDPs and synchronous settings.

\begin{remark} \label{rem-mdp} \rm
In MDPs, where transitions occur at unit intervals, RVI Q-learning simplifies to \eqref{eq: q-ite} with $1$ in place of the holding-time estimates, so stepsizes $\{\beta_n\}$ and Assum.~\ref{cond-rvi-alg2}(ii) are not involved, and $\tm = 1$. For hierarchical control in MDPs as studied in \cite{WYS24}, $\tm \geq 1$. In both cases, $A_* \leq 2 + L_f$, and this bound can be used to set the stepsize scaling parameters when applying Thm.~\ref{thm-rvi-ql2}. No model knowledge is required in these cases. 

Moreover, the global Lipschitz constant $L_f$ in the definition of $A_*$ can be replaced by a local Lipschitz constant of $f$ on a neighborhood of the compact set $\Q_f$. This refinement is related to the use of local Lipschitz constants in the shadowing analysis discussed in Rem.~\ref{rmk-local-lip}, as will become clear from the convergence proof in Sec.~\ref{sec-shad2}. We can exploit this observation in choosing the function $f$, even though $\Q_f$ is unknown. To illustrate, consider the function $f=\psi\circ g$ in Ex.~\ref{ex-f}, where $g(Q)=\max_{s,a}Q(s,a)$ and $\psi$ is the piecewise linear function defined in \eqref{eq-ex-psi} using a known interval $(a,b)$ containing $r^*$. Since $\psi(y)=y$ for $y\in[a,b]$, we have $\Q_f=\Q_g$, and $f$ has local Lipschitz constant $1$ on a neighborhood of $\Q_f$, although its global Lipschitz constant is $\beta>1$. Thus, for MDPs and hierarchical control in MDPs, the improved bound $A_* \leq 3$ can be used when applying Thm.~\ref{thm-rvi-ql2} with this choice of $f$, leading to a lower threshold for the stepsize scaling parameter $A$. \myqed
\end{remark}

\begin{remark} \rm \label{rem-sync}
If the RVI Q-learning algorithm is implemented \emph{synchronously}, then the scaling parameter $A$ for class-2 stepsizes can be chosen \emph{freely} while maintaining the convergence conclusion of Thm.~\ref{thm-rvi-ql2}. This is because, for class-2 stepsizes, the constraint on $A$ arises solely from handling asynchrony; in the synchronous setting, this constraint is removed, without affecting the conclusions of Thm.~\ref{thm-3}, from which Thm.~\ref{thm-rvi-ql2} is derived. (See \cite[Rem.~2.6(a)]{YWS25a}; see also the prior results \cite{Ben96,Ben99} for synchronous SA algorithms.) Moreover, in this setting, if we choose $\beta_n = \frac{1}{B n (\ln n)^b}$ for some $B > 0$ and $b \in [0,1)$, as discussed in Example~\ref{ex-beta}, then Assum.~\ref{cond-rvi-alg2}(ii) is satisfied, and there is no need to estimate bounds on $\tm$ and $A_*$. \myqed
\end{remark}

\section{Convergence Proofs} \label{sec-4}
This section provides the proofs for Thms.~\ref{thm-rvi-ql} and~\ref{thm-rvi-ql2}.

\subsection{Proof of Theorem~\ref{thm-rvi-ql}} \label{sec-4.1}

To prove Thm.~\ref{thm-rvi-ql}, we proceed in two steps. First, in Sec.~\ref{sec-4.1.1}, we rewrite the RVI Q-learning algorithm \eqref{eq: q-ite}-\eqref{eq: t-ite} in the framework of the general SA algorithm \eqref{eq-alg0}, aiming to invoke the results in Sec.~\ref{sec-sa} from \cite{YWS25a}. Then, in Sec.~\ref{sec-4.1.2}, we analyze the solution properties of the corresponding ODEs and combine these with Thm.~\ref{thm-ql} to establish the theorem. 

\subsubsection{Relating RVI Q-Learning to Algorithm \eqref{eq-alg0}} \label{sec-4.1.1}

To define the function $h$ in \eqref{eq-alg0}, similarly to Schweitzer's reasoning for his RVI algorithm \cite{Sch71}, we first fix some $\bar \alpha \in (0, \min_{s \in \S, a \in \A} t_{sa}]$ and define an operator $\T: \R^{|\S \times \A|} \to \R^{|\S \times \A|}$ by  
\begin{equation} \label{eq-T}
    \T(q)(s,a) \= \frac{\bar \alpha \, r_{sa}}{t_{sa}}  + \frac{\bar \alpha}{t_{sa}} \cdot \sum_{s' \in \S} p_{ss'}^a \max_{a' \in \A} q(s', a') + \Big(1 - \frac{\bar \alpha}{t_{sa}}\Big) \cdot q(s,a), \quad s \in \S, \, a \in \A.
\end{equation}
We then define $h : \R^{|\S \times \A|} \to \R^{|\S \times \A|}$ as: for $s \in \S$ and $a \in \A$,
\begin{align}
  h(q)(s,a) & \=  \bar \alpha \cdot \left( \frac{r_{sa}  +  \sum_{s' \in \S} p_{ss'}^a \max_{a' \in \A} q(s', a') - q(s,a)}{t_{sa}} - f(q) \right) \label{eq-def-h} \\
  & \, = \T(q)(s,a) - q(s,a) - \bar \alpha f(q). \label{eq-alt-h}
\end{align}

Before proceeding, several properties of $\T$ are worth noting: 
\begin{enumerate}[leftmargin=0.75cm,labelwidth=!] 
\item[(i)] Since $0 < \frac{\bar \alpha}{t_{sa}} \leq 1$ for all $(s,a)$ by the choice of $\bar \alpha$, $\T$ is nonexpansive w.r.t.\ $\|\cdot\|_\infty$. Indeed, $\T$ can be viewed as the dynamic programming operator for an MDP---this transformation of an SMDP into an equivalent MDP is introduced by Schweitzer in deriving his RVI algorithm \cite{Sch71}.
\item[(ii)] For $c \in \R$ and $q \in \R^{|\S \times \A|}$, $\T(q + c) = \T(q) + c$.
\item[(iii)] The following equation is equivalent to AOE \eqref{eq-opt} with $\bar r = r^*$:
\begin{equation} \label{eq-def-h'}
  h'(q) \= \T(q) - q - \bar \alpha \, r^* = 0,
\end{equation}
and by (ii), the function $h'$ just defined satisfies $h'(q+c) = h'(q)$ for $c \in \R$.
\end{enumerate}
Also, define an operator $\T^o$ as in \eqref{eq-T} but with the expected rewards $r_{sa}$ all set to $0$; that is,
$$ \T^o(q)(s,a) \= \frac{\bar \alpha}{t_{sa}} \cdot \sum_{s' \in \S} p_{ss'}^a \max_{a' \in \A} q(s', a') + \Big(1 - \frac{\bar \alpha}{t_{sa}}\Big) \cdot q(s,a), \quad s \in \S, \, a \in \A.$$
Then $\T^o$ has the same properties (i,\,ii), and the equation $\T^o(q) - q = 0$ corresponds to AOE \eqref{eq-opt} in the case of zero rewards.

Recall the notation $E_g \= \{ x \in \R^d \mid g(x) = 0\}$ for $g : \R^d \to \R^d$. The next lemma 
partially verifies Assum.~\ref{cond-h} on $h$ (Assum.~\ref{cond-h}(iii) will be verified later in Prop.~\ref{prp-rvi-asyn-stab0}).

\begin{lemma} \label{lem-cond-h}
Under Assums.~\ref{cond-smdp} and~\ref{cond-f}, the function $h$ defined by \eqref{eq-def-h} satisfies Assum.~\ref{cond-h}(i,\,ii) with $h_\infty$ given by $h_\infty(q) =  \T^o(q) - q - \bar \alpha f_\infty(q)$. Furthermore, when the SMDP is weakly communicating, $E_h = \Q_f$ and $E_{h_\infty} = \Q^o_{f_\infty} = \{0\}$, while $E_{h'} = \Q$ for the function $h'$ defined in \eqref{eq-def-h'}.
\end{lemma} 

\begin{proof}
The first statement holds by the definition of $h$ and Assum.~\ref{cond-f} on $f$. We have $E_h = \Q_f$ and $E_{h_\infty} = \Q^o_{f_\infty}$, since the equation $h(q) = 0$ is equivalent to \eqref{eq-for-rvi}, while the equation $h_\infty(q) = 0$ is equivalent to \eqref{eq-for-rvi0}. That $E_{h'} = \Q$ follows from the equivalence of $h'(q) = 0$ to AOE \eqref{eq-opt}. Finally, $E_{h_\infty} = \{0\}$ by Prop.~\ref{prp-sol-set}.
\end{proof} 

We now express the RVI Q-learning algorithm \eqref{eq: q-ite}--\eqref{eq: t-ite} in the form of the SA algorithm \eqref{eq-alg0}: for all $i = (s,a) \in \S \times \A$,
\begin{equation} \label{eq-rvi-sa-form}
 Q_{n+1}(i)  = Q_n(i)  + \frac{\alpha_{\nu(n,i)}}{\bar \alpha} \cdot \big( h(Q_n)(i) + M_{n+1}(i) + \epsilon_{n+1}(i) \big) \cdot \ind \{ i \in Y_n\},
\end{equation}  
where we define the noise terms $M_{n+1}$ and $\epsilon_{n+1}$ as follows.
 For each $i = (s,a) \in Y_n$, 
\begin{align} 
    M_{n+1}(i) &= \bar \alpha \cdot \bigg( \frac{R_{n+1}^{sa} - r_{sa}}{T_n(s, a) \vee \eta_n}  \notag \\ 
    & \qquad \quad \ \,  
      +  \frac{\max_{a' \in \A} Q_n(S_{n+1}^{sa}, a') - \sum_{s' \in \S} p_{ss'}^a \max_{a' \in \A} Q_n(s', a')}{t_{sa}} \bigg), \label{eq: M}\\
    \epsilon_{n+1}(i) &= \bar \alpha \cdot \bigg(  \frac{r_{sa} + \max_{a' \in \A} Q_n(S_{n+1}^{sa}, a') - Q_n(s, a) }{T_n(s, a) \vee \eta_n}   \notag \\ 
    & \qquad \quad \ \, - \frac{r_{sa} + \max_{a' \in \A} Q_n(S_{n+1}^{sa}, a') - Q_n(s, a)}{t_{sa}} \bigg), \label{eq: epsilon}
\end{align}
while $M_{n+1}(i) = \epsilon_{n+1}(i) = 0$ if $i \not\in Y_n$. Let $\F_n = \sigma(Q_m, T_m, Y_m, M_m, \epsilon_m; m \leq n)$.

\begin{lemma} \label{lem-cond-noise}
Under Assums.~\ref{cond-smdp}--\ref{cond-f}, $\{M_{n}\}$ and $\{\epsilon_n\}$ satisfy Assum.~\ref{cond-ns}.
\end{lemma} 

\begin{proof}
Under our assumptions, $f$ is Lipschitz continuous, the expected holding times $t_{sa} > 0$, and all the random rewards $R^{sa}_{n+1}$ have finite variances. By definition, the constants $\eta_n > 0$. Using these facts, we obtain $\E [\| Q_n \| ] < \infty$ from \eqref{eq: q-ite}, and then $\E [\| M_{n+1} \| ] < \infty$ from \eqref{eq: M}, for all $n \geq 0$. By \eqref{eq: M} and how the transition data $(R_{n+1}^{sa}, S_{n+1}^{sa})$ are generated, $\E [ M_{n+1} \mid \F_n ] = 0$ a.s. 

To verify the remaining conditions in Assum.~\ref{cond-ns}, note first that Assums.~\ref{cond-smdp}(ii) and~\ref{cond-rvi-alg}(i,\,ii) ensure the convergence $T_n(s,a) \to t_{sa}$ a.s.\ for each $(s, a) \in \S \times \A$, by applying standard SA theory \cite{Bor23,KuY03} to the stochastic-gradient-descent update rule \eqref{eq: t-ite}.
By the definitions of $M_{n+1}$ and $\epsilon_{n+1}$ in \eqref{eq: M}--\eqref{eq: epsilon}, direct calculations yield the following bounds: for some suitable constants $\bar K, \bar K' > 0$, 
\begin{align*}
  \E [ \| M_{n+1} \|^2 \mid \F_n ] & \leq  \underset{K_n}{\underbrace{\left( 1 \vee \max_{(s, a) \in \S \times \A} \big(T_n(s, a) \vee \eta_n \big)^{-2} \right) \cdot \bar K}} \cdot ( 1 + \| Q_n\|^2) \ \ \ a.s., \\
   \| \epsilon_{n+1} \| & \leq  \underset{\delta_{n+1}}{\underbrace{ \left( \max_{(s, a) \in \S \times \A} \Big| \big(T_n(s, a) \vee \eta_n \big)^{-1} - t_{sa}^{-1} \Big| \right) \cdot \bar K' }} \cdot ( 1 + \| Q_n\|) \ \ \ a.s. 
\end{align*}
Both $K_n$ and $\delta_{n+1}$, being $\F_n$-measurable, satisfy the required measurability.
 Since $T_n(s,a) \to t_{sa} > 0$ a.s.\ and by Assum.~\ref{cond-rvi-alg} $\eta_n \to 0$, whence $T_n(s,a) \vee \eta_n \to t_{sa}$ a.s., we also have $\sup_n K_n < \infty$ and $\delta_n \to 0$ a.s., as required. 
\end{proof}

\begin{remark}[{\normalfont On relaxing the finite-variance assumption
for the SMDP model}] \label{rem-relax-variance-cond} \rm
This lemma and its proof remain valid if, for each state-action pair $(s,a)$, the random holding times
$\tau^{sa}_{n+1}$ have finite mean but infinite variance, provided that the holding time estimates $T_n(s,a) \to t_{sa}$ a.s. This can be achieved by using sample means, whose a.s.\ convergence follows from the strong law of large numbers for i.i.d.\ random variables. Accordingly, the convergence conclusions of Thm.~\ref{thm-rvi-ql} continue to hold under this relaxed condition on the holding time distributions.

However, when the random holding times can have infinite variance, the random rewards associated with the transitions may also have infinite variance. In this case, $\{M_n\}$ as defined above may violate the conditional variance conditions in Assum.~\ref{cond-ns}(i). Addressing this more general setting may therefore require algorithmic changes, such as using a.s.\ convergent estimates of the expected rewards $r_{sa}$ in place of the random rewards $R^{sa}_{n+1}$ in the RVI Q-learning updates. The resulting estimation errors can then be absorbed into the biased noise terms $\epsilon_{n+1}$, so that the preceding lemma continues to hold and the SA results remain applicable as in the present case. \myqed
\end{remark}

\subsubsection{Analyzing Solution Properties of Associated ODEs} \label{sec-4.1.2}

With the initial steps completed, we next aim to establish the global asymptotic stability of $\Q_f$ and verify the required stability criterion of Borkar and Meyn---in particular, Assum.~\ref{cond-h}(iii)---through the following propositions:

\begin{prop} \label{prp-rvi-asyn-stab}
Under the conditions of Thm.~\ref{thm-rvi-ql}, with $h$ given by \eqref{eq-def-h}, the set $E_h = \Q_f$ is globally asymptotically stable for the ODE $\dot x(t) = h(x(t))$.
\end{prop}

\begin{prop} \label{prp-rvi-asyn-stab0}
Under the conditions of Prop.~\ref{prp-rvi-asyn-stab}, the origin is the unique globally asymptotically stable equilibrium of the ODE $\dot x(t) = h_\infty(x(t))$.
\end{prop}

Proposition~\ref{prp-rvi-asyn-stab0} will be proved by specializing our proof arguments for Prop.~\ref{prp-rvi-asyn-stab} to the case where the SMDP has zero rewards. Assuming both propositions have been proved, we can then invoke Thm.~\ref{thm-ql} to derive Thm.~\ref{thm-rvi-ql} as follows:

\begin{proof}[Proof of Thm.~\ref{thm-rvi-ql}]
Consider the RVI Q-learning algorithm in its equivalent form \eqref{eq-rvi-sa-form}. We verify one by one that the conditions of Thm.~\ref{thm-ql} are met. First, by Lem.~\ref{lem-cond-h} and Prop.~\ref{prp-rvi-asyn-stab0}, $h$ satisfies Assum.~\ref{cond-h}, with $E_h = \Q_f$. Assumption~\ref{cond-ns} holds by Lem.~\ref{lem-cond-noise}. The scaled stepsizes $\{\alpha_n/\bar \alpha\}$ and the asynchronous update scheme satisfy Assums.~\ref{cond-ss}--\ref{cond-us} due to the algorithmic requirements in Assum.~\ref{cond-rvi-alg}.
The remaining condition is that $E_h$ is globally asymptotically stable for the ODE $\dot x(t) = h(x(t))$, which follows from Prop.~\ref{prp-rvi-asyn-stab}. The desired a.s.\ convergence properties of $\{Q_n\}$ stated in Thm.~\ref{thm-rvi-ql} now follow from Thm.~\ref{thm-ql}. The convergence $f(Q_n) \to r^*$ a.s.\ follows from the a.s.\ convergence of $Q_n$ to $\Q_f$, the continuity of $f$, and the definition of $\Q_f$.\end{proof}

Now, let us proceed to prove Prop.~\ref{prp-rvi-asyn-stab}. Similarly to the approach in Abounadi et al.\ \cite[Sec.\ 3.1]{ABB01}, we consider two ODEs, $\dot x(t) = h(x(t))$ and $\dot{y}(t) = h'(y(t))$, for $h$ and $h'$ defined in \eqref{eq-def-h} and \eqref{eq-def-h'}, respectively:
\begin{align}    
\dot x(t) & = h(x(t)), \quad \ \text{where} \ h(x) = \T(x) - x - \bar \alpha f(x), \label{eq: ode0} \\
 \dot{y}(t) & = h'(y(t)), \quad \,\text{where} \ h'(y) = \T(y) - y - \bar \alpha \, r^*. \label{eq: aux-ode}
\end{align}
We relate the solutions of \eqref{eq: ode0} to those of \eqref{eq: aux-ode}, whose asymptotic properties follow directly from Borkar and Soumyanath \cite{BoS97}. Recall that $\T$ is nonexpansive w.r.t.\ $\|\cdot\|_\infty$ and $E_{h'} = \Q \not= \varnothing$ (Lem.~\ref{lem-cond-h}). By \cite{BoS97}, solutions of \eqref{eq: aux-ode} have the following important properties, where we emphasize that \emph{we henceforth write $\| \cdot \| = \| \cdot \|_\infty$ and use this notation throughout the proofs below.}

\begin{lemma}[from {\cite[Thm.\ 3.1 and Lem.\ 3.2]{BoS97}}] \label{lem: aux-ode}
For any solution $y(\cdot)$ of the ODE \eqref{eq: aux-ode} and any $\bar y \in E_{h'} = \Q$, 
the distance $\|y(t) - \bar y \|$ is nonincreasing, and as $t \to \infty$, $y(t) \to y_\infty$, a point in $E_{h'}$ 
that may depend on $y(0)$.
\end{lemma}
 
\begin{remark} \rm \label{rmk-novel-proof}
From this point on, our proof arguments depart significantly from the arguments employed in \cite{ABB01} and their recent extensions in \cite{WNS21a,WYS24} for proving similar conclusions. The earlier analyses utilized an explicit expression of the difference $x(t) - y(t)$ (in terms of $y(t)$), obtained via the variation of constants formula, for a specific family of functions $f$ (cf.\ Ex.~\ref{ex-f}). As we deal with a more general family of $f$ under Assum.~\ref{cond-f}, explicit expressions of $x(t) - y(t)$ are not available. Instead, to characterize this difference in Lem.~\ref{lem-ode0-aux-ode} below, we utilize an existence/uniqueness theorem for non-autonomous ODEs. Subsequently, to derive asymptotic properties of $x(t)$ in Lems.~\ref{lem-main-1} and~\ref{lem-main-2}, we make extensive use of the SISTr properties of $f$ and $f_\infty$. 

It is also worth noting that, in \cite{ABB01}, the proof showing that $x(t) - y(t)$ is a constant vector relies on the nonexpansiveness of the operator $\T$ w.r.t.\ the span seminorm. Our proofs do not use this property, and as such, they are potentially applicable to more general problem settings beyond SMDPs/MDPs. \myqed
\end{remark} 

\begin{lemma}\label{lem-ode0-aux-ode}
Let $x(t)$ and $y(t)$ be solutions of the ODEs \eqref{eq: ode0} and \eqref{eq: aux-ode}, respectively, with the same initial condition $x(0) = y(0)$. Then $x(t)= y(t) + z(t)$ for all $t \in \R$, where $z(t)$ is the unique real-valued solution to
the scalar ODE \begin{equation} \label{eq-z-prf1}
    \dot{z}(t) = \bar \alpha \, r^* - \bar \alpha f(y(t) + z(t)), \quad \text{with} \ z(0) = 0.
\end{equation} 
\end{lemma}

\begin{proof} 
Consider a function $\phi(t) = y(t) + z(t)$, where $z(t)$ is some real-valued differentiable function with $z(0) = 0$. If $\phi$ satisfies the ODE \eqref{eq: ode0}, then $x(\cdot) = \phi(\cdot)$ since \eqref{eq: ode0} has a unique solution for each initial condition.
Since $\T(y(t) + z(t))  = \T(y(t)) + z(t)$ and $\dot{y}(t)=h'(y(t))$, for $\phi$ to satisfy \eqref{eq: ode0} [i.e., $\dot{y}(t) + \dot{z}(t) = h(y(t) + z(t))$], it is equivalent, by the definitions of $h$ and $h'$, to having $z(t)$ satisfy \eqref{eq-z-prf1}. 

Now express \eqref{eq-z-prf1} as the non-autonomous ODE $\dot{z}(t) = \psi(t, z(t))$, where $\psi$ is the continuous function $\psi(t, z) \= \bar \alpha \, r^* - \bar \alpha f(y(t) + z )$ defined on $\R \times \R$. By the Lipschitz continuity of $f$ (Assum.~\ref{cond-f}(i)), $\psi$ is Lipschitz continuous in $z$ uniformly w.r.t.\ $t$. Then, by the Picard-Lindel\"{o}f existence theorem and an extension theorem for ODEs \citep[Chap.\ II, Thms.~1.1 and 3.1]{Har02}, a unique solution to \eqref{eq-z-prf1} exists on a maximal interval $(t_-, t_+)$, with $(t, z(t))$ approaching the boundary of $\R \times \R$ as $t \to t_-$ or $t \to t_+$. Since $f$ is Lipschitz continuous and $y(t)$, being continuous on $\R$, is bounded on bounded intervals, Gronwall's inequality implies that $|z(t)|$ cannot tend to $\infty$ as $t$ approaches a finite limit. This forces $(t_-, t_+) = (- \infty, \infty)$, completing the proof.
\end{proof}

\begin{remark} \label{rem-asym-auto} \rm
A previous anonymous reviewer noted an important connection to asymptotically autonomous systems: By Lem.~\ref{lem: aux-ode}, the non-autonomous ODE~\eqref{eq-z-prf1} is \emph{asymptotically autonomous}, with limiting ODE $\dot{z}(t) = \bar \alpha \, r^* - \bar \alpha f(\bar y + z(t))$, where $\bar y = \lim_{t\to\infty} y(t)$. 
The qualitative theory of these systems is well developed \cite{Mar56,MST95,Thi94}, although general results typically assume bounded (forward) solution trajectories. Based on this theory (see, e.g., \cite[Thm.~1.1]{Thi94}), if $z(t)$, $t \geq 0$, is a bounded solution of~\eqref{eq-z-prf1}, then, in view of the SISTr property of $f$ and Lem.~\ref{lem-f}, $z(t)$ must converge to the unique equilibrium point $\bar z$ of the limiting ODE, characterized by $f(\bar y + \bar z) = r^*$. Rather than first proving boundedness and then invoking the general theory, we give direct proofs below using Lyapunov-type arguments based on a Lyapunov function for the limiting ODE. We find this approach simpler and more direct in this case.

We also note that the viewpoint of asymptotically autonomous systems can help broaden the admissible class of functions $f$ beyond those satisfying the SISTr conditions considered here. \myqed
\end{remark}

\begin{lemma} \label{lem-main-1}
For any solution $x(\cdot)$ of the ODE \eqref{eq: ode0}, as $t \to \infty$, $x(t)$ converges to a point in $E_h = \Q_f$ that may depend on $x(0)$.
\end{lemma}

\begin{proof} 
Let $y(\cdot)$ be the solution of \eqref{eq: aux-ode} with $y(0) = x(0)$. Express $x(t)$ as $x(t) = y(t) + z(t)$ according to Lem.~\ref{lem-ode0-aux-ode}. By Lem.~\ref{lem: aux-ode}, $\lim_{t \to \infty} y(t) = : \bar y \in E_{h'} = \Q$.
By Lem.~\ref{lem-f}, there is a unique scalar $\bar z$ satisfying $f(\bar y + \bar z) = r^*$. Let us show that  as $t \to \infty$, $z(t) \to \bar z$, which will imply $x(t) \to \bar y + \bar z \in \Q_f$.

Define $V_{\bar y} : \R \to \R_+$ by $V_{\bar y}(z) \= \frac{1}{2 \bar \alpha}(z - \bar z)^2$.
By Lem.~\ref{lem-ode0-aux-ode} and the fact $f(\bar y + \bar z) = r^*$, the time derivative $\dot{V}_{\bar y}(z(\cdot))$ satisfies
\begin{align}
   \dot{V}_{\bar y}(z(t)) & = (z(t) - \bar z) \cdot  \big( r^* - f( y(t) + z(t)) \big) \notag \\
      & =  (z(t) - \bar z) \cdot \big( f(\bar y + \bar z) - f(\bar y + z(t)) + f(\bar y + z(t)) - f( y(t) + z(t)) \big) \notag \\
      & \leq  (z(t) - \bar z) \cdot \big( f(\bar y + \bar z) - f(\bar y + z(t)) \big) +   | z(t) - \bar z| \cdot L_f \| \bar y - y(t) \|.  \label{eq-lem-m1-prf2}
\end{align}
Let $\delta > 0$. Since $f$ is SISTr  
(Assum.\ \ref{cond-f}), there exists $\epsilon_{\bar y} > 0$ such that $f(\bar y + z ) \geq f(\bar y + \bar z) + \epsilon_{\bar y}$ if $z - \bar z \geq \delta$, and $f(\bar y + z ) \leq f(\bar y + \bar z) - \epsilon_{\bar y}$ if $z - \bar z \leq - \delta$. Consequently,
\begin{equation} \label{eq-lem-m1-prf3}
  | z - \bar z | \geq \delta \quad \overset{\text{implies}}{\Longrightarrow} \quad  (z - \bar z) \cdot \big( f(\bar y + \bar z) - f(\bar y + z ) \big) \leq - \epsilon_{\bar y} | z - \bar z |.
\end{equation}   
Using \eqref{eq-lem-m1-prf3} and the fact that $\| \bar y - y(t) \| \to 0$ as $t \to \infty$, it follows from \eqref{eq-lem-m1-prf2} that for some $t_0$ sufficiently large so that $L_f \| \bar y - y(t)\| < \epsilon_{\bar y}/2$ for all $t \geq t_0$, we have that
$$ | z(t) - \bar z | \geq \delta, \ t \geq t_0 \quad \overset{\text{implies}}{\Longrightarrow} \quad \dot{V}_{\bar y}(z(t)) \leq - \tfrac{ \epsilon_{\bar y}}{2} | z(t) - \bar z | \leq - \tfrac{ \epsilon_{\bar y} \delta}{2} < 0.$$
This implies that there exists some finite time $t_{\delta} \geq t_0$ such that $|z(t) - \bar z | \leq \delta$ for all $t \geq t_\delta$. Since $\delta$ is arbitrary, we obtain $z(t) \to \bar z$ as $t \to \infty$.
\end{proof}

Next, we prove that $\Q_f$ is Lyapunov stable for the ODE \eqref{eq: ode0}, which will establish its global asymptotic stability in view of Lem.~\ref{lem-main-1}. For this part of the proof, we need another implication of the SISTr property of $f$, given in the lemma below:

\begin{lemma} \label{lem-f2}
Assumption~\ref{cond-f}(i) on $f$ implies the following:
For $\delta > 0$ and $x \in \R^d$, let $\epsilon_{x, \delta} > 0$ be the smallest number with $\min \big\{ f(x + \epsilon_{x, \delta}) - f(x), \, f(x) - f(x - \epsilon_{x, \delta}) \big\} =  \delta$. Then for any bounded set $D \subset \R^d$, $\sup_{x \in D} \epsilon_{x, \delta} < \infty$ and $\sup_{x \in D} \epsilon_{x, \delta} \downarrow 0$ as $\delta \downarrow 0$.
\end{lemma}

\begin{proof}
Since the function $f$ is SISTr, $\epsilon_{x,\delta} > 0$ is well-defined, finite, and nondecreasing in $\delta$. 
To show $\sup_{x \in D} \epsilon_{x, \delta} < \infty$, suppose, for contradiction, that there exists a sequence $\{x_n\} \subset D$ with $\epsilon_n \= \epsilon_{x_n, \delta} \to \infty$. Since $D$ is bounded, by passing to a subsequence if necessary, we can assume $x_n \to \bar x \in \R^d$. Let $\kappa_n = 1$ or $-1$ depending on whether $f(x_n + \epsilon_{n}) - f(x_n) \leq f(x_n) - f(x_n - \epsilon_n)$ or not. Then $| f(x_n + \kappa_n \epsilon_n) - f(x_n) | = \delta$ for all $n$, while by the Lipschitz continuity of $f$, we have $f(x_n) \to f(\bar x)$ and $|f(x_n + \kappa_n \epsilon_n) - f(\bar x + \kappa_n \epsilon_n) | \to 0$ as $n \to \infty$. These relations together imply $|f(\bar x + \kappa_n \epsilon_n) - f(\bar x) | \to \delta$ as $n \to \infty$. However, since $f$ is SISTr, $|\kappa_n \epsilon_n | \to \infty$ implies $|f(\bar x + \kappa_n \epsilon_n)| \to \infty$. This contradiction proves that $\sup_{x \in D} \epsilon_{x, \delta} < \infty$.

As $\delta \downarrow 0$, $\sup_{x \in D} \epsilon_{x, \delta}$ is nonincreasing; 
suppose for the sake of contradiction that $\sup_{x \in D} \epsilon_{x, \delta} \not\to 0$. 
Then, similarly to the above argument, there exist a sequence $\delta_n \downarrow 0$ and a sequence $\{x_n\} \subset D$ such that $x_n \to \bar x \in \R^d$, $\epsilon_n \= \epsilon_{x_n, \delta_n} \to \bar \epsilon > 0$, and 
$|f(x_n) - f(x_n +  \kappa  \epsilon_n)| = \delta_n,$ 
where $\kappa = 1$ or $-1$. Letting $n \to \infty$ yields $f(\bar x) = f(\bar x + \kappa \bar \epsilon)$, which is impossible since $f$ is SISTr. This proves that $\sup_{x \in D} \epsilon_{x, \delta} \downarrow 0$ as $\delta \downarrow 0$.
\end{proof}

\begin{lemma} \label{lem-main-2}
The set $E_h = \Q_f$ is Lyapunov stable for the ODE \eqref{eq: ode0}.
\end{lemma}

\begin{proof}
First, consider a solution $x(t)$ of \eqref{eq: ode0} that starts in some $\delta_0$-neighborhood of $\Q_f$, where $\delta_0 > 0$, and let $\bar x \in \Q_f$ be such that $\|x(0) - \bar x\| \leq \delta_0$. We derive a bound on $\|x(t) - \bar x\|$ for $t \geq 0$ in terms of $\delta_0$. Decompose $x(t)$ as $x(t) = y(t) + z(t)$ according to Lem.~\ref{lem-ode0-aux-ode}, with $y(0) = x(0)$ and $z(0) = 0$. 
Since $\Q_f \subset E_{h'} =\Q$, Lem.~\ref{lem: aux-ode} implies that $\| y(t) - \bar x\|$ is nonincreasing in $t$. Consequently,
\begin{equation} \label{eq-m2-prf0}
   \| x(t) - \bar x \| \leq \| y(t) - \bar x \| + | z(t) | \leq \delta_0 + | z(t)|, \qquad \forall \, t \geq 0.
\end{equation}
To bound the term $| z(t)|$, define $V_{\bar x} : \R \to \R_+$ by $V_{\bar x}(z) \= \frac{z^2}{2 \bar \alpha}$. Since $f(\bar x) = r^*$, similarly to the derivation of \eqref{eq-lem-m1-prf2} (with $\bar x$ in place of $\bar y$ and $\bar z = 0$), we have 
\begin{equation}
\dot{V}_{\bar x}(z(t)))  \leq z(t) \cdot \big( f(\bar x) - f(\bar x + z(t)) \big) +   | z(t) | \cdot L_f \| \bar x - y(t) \|.   \label{eq-lem-m2-prf1}
\end{equation}
Let $\delta \= (L_f +1) \delta_0$, and let $\epsilon_{\bar x, \delta} > 0$ be as defined in Lem.~\ref{lem-f2}. Since $f$ is SISTr, similarly to the derivation of \eqref{eq-lem-m1-prf3}, we have
\begin{equation}
  | z(t) | \geq \epsilon_{\bar x, \delta}  \quad \overset{\text{implies}}{\Longrightarrow} \quad z(t) \cdot \big( f(\bar x) - f(\bar x + z(t)) \big)  \leq - \delta |z(t)|, \notag
\end{equation} 
and then by \eqref{eq-lem-m2-prf1} and the fact $\| \bar x - y(t) \| \leq \delta_0$,
\begin{equation} \label{eq-lem-m2-prf2}
  | z(t) | \geq \epsilon_{\bar x, \delta}  \quad \overset{\text{implies}}{\Longrightarrow} \quad   \dot{V}_{\bar x}(z(t))) \leq - \delta_0 | z(t)| \leq - \delta_0 \epsilon_{\bar x, \delta}  < 0.
\end{equation}
Since $z(0) = 0$, the relation \eqref{eq-lem-m2-prf2} implies that 
\begin{equation} \label{eq-lem-m2-prf3}
|z(t)| \leq \epsilon_{\bar x, \delta} \ \ \ \forall \, t \geq 0, \quad \text{where} \ \delta = (L_f +1) \delta_0.
\end{equation}

We now use this bound and Lem.~\ref{lem-f2} to prove the Lyapunov stability of $\Q_f$. For any $\epsilon > 0$, by the compactness of $\Q_f$ (Prop.~\ref{prp-sol-set}) and Lem.~\ref{lem-f2}, there exists a sufficiently small $\bar \delta > 0$ such that $\sup_{\bar x \in \Q_f} \epsilon_{\bar x, \delta'} \leq \epsilon/2$ for all $\delta' \leq \bar \delta$. Define $\bar \delta_0 = \frac{\epsilon}{2} \wedge \frac{\bar \delta}{(L_f + 1)}$. Then, by \eqref{eq-m2-prf0} and the bound \eqref{eq-lem-m2-prf3} on $|z(t)|$ with $\delta_0 = \bar \delta_0$, if $x(0)$ lies in the $\bar \delta_0$-neighborhood of $\Q_f$, then $x(t)$ remains in the $\epsilon$-neighborhood of $\Q_f$ for all $t \geq 0$. This establishes that $\Q_f$ is Lyapunov stable. 
\end{proof}

By Lems.~\ref{lem-main-1} and~\ref{lem-main-2}, $\Q_f$ is globally asymptotically stable for the ODE \eqref{eq: ode0}. This proves Prop.\ \ref{prp-rvi-asyn-stab}. 

We now specialize the above proof arguments to establish Prop.~\ref{prp-rvi-asyn-stab0}: 

\begin{proof}[Proof of Prop.~\ref{prp-rvi-asyn-stab0}]
Consider the case where the SMDP has zero rewards, with the functions $f_\infty$ and $h_\infty$ taking the roles of $f$ and $h$, respectively. In this case, $r^* = 0$ and the operator $\T^o$ replaces $\T$. Instead of $\Q$, we have the solution set of AOE given by $\Q^o \= \{ c \1 \,|\, c \in \R\}$ (Lem.~\ref{lem: sol-struc}), while instead of $\Q_f$, we have $\Q^o_{f_\infty} = E_{h_\infty} =  \{0\}$ (Lem.~\ref{lem-cond-h}). Lemmas~\ref{lem: aux-ode} and~\ref{lem-ode0-aux-ode} hold with these substitutions. In particular, they establish that, for each initial condition, the solution of $\dot{x}(t) = h_\infty(x(t))$ can be written as $x(t) = y(t) + z(t)$, where $y(t)$ converges to a constant vector $\bar y \in \Q^o$, and $z(t)$ satisfies $\dot{z}(t) = - \bar \alpha f_\infty(y(t) + z(t))$.  

Next, observe that the proof of Lem.~\ref{lem-main-1} relies only on $f$ being SISTr at every point of $\Q_f$ (see the derivation of \eqref{eq-lem-m1-prf3}, which uses the SISTr property of $f$ at $\bar y + \bar z \in \Q_f$), along with the existence of a unique scalar $\bar z$ solving $f(\bar y + \bar z) = r^*$ for any given $\bar y \in \Q$. In the present case, $\Q_{f_\infty}^o = \{0\}$, and $f_\infty$ is SISTr at the origin by Assum.~\ref{cond-f}(ii). This assumption also implies that, for any constant vector $\bar y = c \1 \in \Q^o$, the equation $f_\infty(\bar y + \bar z) = 0$ has a unique solution $\bar z = -c$. Therefore, the conclusion of Lem.~\ref{lem-main-1} applies here as well and shows that any solution $x(t)$ of the ODE $\dot{x}(t) = h_\infty(x(t))$ converges to the origin as $t \to \infty$.

For the Lyapunov stability of the origin, observe that in the proof of Lem.~\ref{lem-main-2}, when defining $\epsilon_{\bar x, \delta}$ and deriving the relation \eqref{eq-lem-m2-prf3} to bound $| z(t)|$ by $\epsilon_{\bar x, \delta}$, we only used the property that $f$ is SISTr at any point $\bar x \in \Q_f$. We then relied on the argument $\sup_{\bar x \in \Q_f} \epsilon_{\bar x, \delta'} \downarrow 0$ as $\delta' \downarrow 0$ to conclude the proof. In the present case, since $\Q_{f_\infty}^o = \{0\}$ and $f_\infty$ is SISTr at the origin, with $\bar x = 0$, the same bound $| z(t) | \leq \epsilon_{\bar x, \delta}$ holds, and we have $ \epsilon_{\bar x, \delta'} \downarrow 0$ as $\delta' \downarrow 0$. Therefore, the conclusion of Lem.~\ref{lem-main-2} also applies here, establishing the Lyapunov stability of the origin and hence its global asymptotic stability.
\end{proof}

This completes the proof of Thm.~\ref{thm-rvi-ql}. Finally, based on the above proof, we make an observation regarding the relaxation of the SISTr condition on $f$.

\begin{remark}[{\normalfont On relaxing the SISTr condition on $f$}] \rm \label{rmk-rel-sistr}
In the proofs of Lems.~\ref{lem-main-1} and~\ref{lem-main-2}, we only need the key relations \eqref{eq-lem-m1-prf3} and \eqref{eq-lem-m2-prf2} to hold for points $\bar y + \bar z$ and $\bar x$ in $\Q_f$, and the bound $\sup_{\bar x \in \Q_f} \epsilon_{\bar x, \delta'} \downarrow 0$ as $\delta' \downarrow 0$. It is possible to obtain these under a localized version of the SISTr condition. For example, if $r^*$ is known to lie in the range $[a, b]$, then we can require that for each $x \in \R^d$ with $a \leq f(x) \leq b$, the function $c \in \R \mapsto f(x + c)$ is strictly increasing in a certain neighborhood of $0$, without requiring it to be monotonic outside this neighborhood, provided the key relations and uniform bound stated above continue to hold.

An alternative approach to relaxing the SISTr condition, based on the theory of asymptotically autonomous systems, was discussed in Rem.~\ref{rem-asym-auto}. \myqed
\end{remark}

\subsection{Proof of Theorem~\ref{thm-rvi-ql2}} \label{sec-shad2}

We apply Thm.~\ref{thm-3} to prove Thm.~\ref{thm-rvi-ql2} for RVI Q-learning. Most of the conditions required by Thm.~\ref{thm-3} are satisfied by assumption or by the proofs already provided in Sec.~\ref{sec-4.1} when establishing Thm.~\ref{thm-rvi-ql}. In particular, the conditions on the associated ODE hold by Prop.~\ref{prp-rvi-asyn-stab} and Lem.~\ref{lem-main-1}. Below, we verify the remaining conditions: the thresholds for the scaling parameters in class-1 or class-2 stepsizes $\{\alpha_n\}$, Assum.~\ref{cond-mns} on the biased noise terms $\{\epsilon_n\}$, and, for class-1 stepsizes,  the additional asynchrony condition Assum.~\ref{cond-mus}. Among these, the verification of Assum.~\ref{cond-mns} will constitute most of the proof. 

Before proceeding, we clarify that all sample path arguments below are understood to be for a path on which the a.s.\ conditions of Thm.~\ref{thm-rvi-ql2} and all previously established a.s.\ properties hold. Any new a.s.\ properties proved below apply to this path from the point they are established. For conciseness, we do not repeat ‘a.s.’ in every instance.

For RVI Q-learning in SMDPs, the biased noise terms $\{\epsilon_n\}$ arise from estimation errors in estimating the expected holding times from data (see \eqref{eq: epsilon}). We first show that under Assum.~\ref{cond-rvi-alg2} on the stepsizes $\{\beta_n\}$ used in estimation, Assum.~\ref{cond-mns} is satisfied by $\{\epsilon_n\}$.

We set the parameter $\bar \alpha \= \tm = \min_{(s,a) \in \S \times \A} t_{sa}$ for the operator $\T$ and the function $h$ in our preceding analysis of RVI Q-learning (see \eqref{eq-T}-\eqref{eq-alt-h}). This choice gives 
$$h(q) = \T(q) - q - \tm f(q),$$
where $\T$ is nonexpansive w.r.t.\ $\| \cdot\|_\infty$, as previously noted. Moreover, $\{\tm^{-1} \alpha_n\}$ serves as the actual stepsize sequence when applying the SA results (see \eqref{eq-rvi-sa-form}), with its scaling parameter given by $\tm A$, where $A$ is the scaling parameter in $\{\alpha_n\}$. 
Observe that 
\begin{equation} \label{eq-prf-shad-rvi0a}
  L_h \leq 2 + \tm L_f, \qquad \ell(\{\tm^{-1} \alpha_n\}) = \tm \ell(\{\alpha_n\}),
\end{equation}  
where, as recalled, $L_h$ and $L_f$ are the Lipschitz constants of $h$ and $f$ under $\|\cdot\|_\infty$, and $\ell(\{\alpha_n\}) = \limsup_{n \to \infty} \tfrac{\ln \alpha_n }{\sum_{k=0}^{n} \alpha_k}$.

Accordingly, the requirement in Thm.~\ref{thm-3} on the noise terms $\{\epsilon_n\}$---namely, that Assum.~\ref{cond-mns} holds with $\mu_\delta < - L_h$---becomes 
\begin{equation} \label{eq-prf-shad-rvi0b}
  \limsup_{n \to \infty} \tfrac{\ln \delta_{n+1} }{\tm^{-1} \sum_{k=0}^n \alpha_k} \leq \mu_\delta < - L_h, \ \ \  \text{for} \ \delta_{n+1} = \bar{K}'\!\! \max_{(s, a) \in \S \times \A} \big| \big(T_n(s, a) \vee \eta_n \big)^{-1} - t_{sa}^{-1} \big|,
\end{equation}  
where $\delta_{n+1}$ is as defined in the proof of~Lem.~\ref{lem-cond-noise}, and $\bar{K}'$ is a constant. Since $T_n(s,a) \to t_{sa}$ and $\eta_n \to 0$, we can bound $\delta_{n+1}$ for sufficiently large $n$ as 
\begin{equation} \label{eq-prf-shad-rvi0}
  \textstyle{\delta_{n+1} \leq c \max_{(s, a) \in \S \times \A} \big| T_n(s, a)  - t_{sa} \big|}, \ \ \ \text{for some constant $c$}.
\end{equation}  
Recall that the estimates $T_n(s, a)$ of the holding time $t_{sa}$ are updated using stochastic gradient descent with a stepsize sequence $\{\beta_n\}$ such that, among other requirements, $\beta_n \geq \varsigma \alpha_n$ eventually for some $\varsigma > 0$ [see \eqref{eq: t-ite} and Assum.~\ref{cond-rvi-alg2}(ii)]. 

\begin{lemma} \label{lem-shad-rvi} Almost surely,
$\limsup_{n \to \infty} \frac{\ln \delta_{n+1} }{\sum_{k=0}^n \alpha_k} \leq \varsigma \max \big\{ \tfrac{\ell(\{\beta_n\})}{2}, \, - 1 \big\}$. 
\end{lemma}

\begin{proof}
For each $(s,a) \in \S \times \A$, consider the iterates $\{T_n(s,a)\}_{n \geq 0}$ in \eqref{eq: t-ite} at the iterations $m_k$ (indexed by $k \geq 0$) when the $(s,a)$-component is updated. Under Assums.~\ref{cond-smdp}(ii) and~\ref{cond-rvi-alg2}(ii), the sequence $\{T_{m_k}(s,a)\}_{k \geq 0}$ follows standard stochastic gradient descent to minimize the objective function $V(y) = \tfrac{1}{2} (y - t_{sa})^2$, using the stepsizes $\{\beta_k\}$:
\begin{equation} \label{eq-prf-Tmk}
 T_{m_{k+1}}(s,a) = T_{m_k+1}(s,a) = T_{m_k}(s, a) + \beta_{k} (\tau_{m_k+1}^{sa} - T_{m_k}(s, a)), \qquad k \geq 0.
\end{equation} 
Define a continuous trajectory $\bar y(t)$ via piecewise linear interpolation of the iterates $\{T_{m_k}(s,a)\}_{k \geq 0}$, using $\{\beta_k\}$ as the elapsed times between consecutive iterates to define the ODE-time (analogously to the construction of the trajectory \eqref{eq-cont-traj2}). Specifically, for $k \geq 0$, let $\hat t(k) \= \sum_{j=0}^{k-1} \beta_j$ with $\hat t(0) \= 0$, and let
\begin{equation} \label{eq-prf-Tmk-traj}
\bar y(t) \=  T_{m_k}(s,a) +  \tfrac{t - \hat t(k)}{\hat t(k+1) - \hat t(k)} \, ( T_{m_{k+1}}(s,a) - T_{m_k}(s,a)), \quad  t \in [\hat t(k), \hat t(k+1)].
\end{equation}
By standard SA theory (see, e.g., \cite{Bor23,KuY03}), the trajectory $\bar y(t)$ is a.s.\ bounded, asymptotically tracks solutions of the scalar ODE $\dot{y}(t) = - (y(t) - t_{sa})$, and converges to $t_{sa}$ as $t \to \infty$.

We next obtain a bound on the asymptotic rate of convergence, w.r.t.\ the ODE-time defined above by $\{\beta_k\}$, using the results of \citet[Sec.\ 8.1]{Ben99}. In particular, since the variances of $\tau_{m_k+1}^{sa}, k \geq 0,$ are bounded by Assum.~\ref{cond-smdp}(ii) and the iterates $\{T_{m_{k+1}}(s,a)\}_{k \geq 0}$ are a.s.\ bounded, the iteration \eqref{eq-prf-Tmk} is an instance of the Robbins-Monro algorithm analyzed in \cite[Prop.\ 8.3]{Ben99}. Its associated ODE $\dot{y}(t) = - (y(t) - t_{sa})$ exhibits exponential convergence to $t_{sa}$ at rate $-1$: for any initial condition $y(0)$, $|y(t) - t_{sa}| = e^{-t} | y(0) - t_{sa}|$ for all $t \geq 0$. This implies, by \cite[Prop.\ 8.3 and Lem. 8.7(i)]{Ben99}, the following bound on the asymptotic convergence rate of $\bar y(t)$ to $t_{sa}$:
\begin{equation} \label{eq-prf-shad-rvi1}
   \limsup_{t \to \infty}  \tfrac{1}{t} \ln \big( | \bar y(t) - t_{sa} | \big) \leq \mu \= \max \big\{ \tfrac{\ell(\{\beta_k\})}{2}, \, - 1 \big\} \ \ a.s.,
\end{equation}
where $\mu < 0$ because $\ell(\{\beta_k\}) < 0$ under Assum.~\ref{cond-rvi-alg2}(ii).

The bound \eqref{eq-prf-shad-rvi1} holds for all components $(s,a) \in \S \times \A$, each with its own local ODE-time $t$ as defined in \eqref{eq-prf-Tmk-traj}. This implies that, w.r.t.\ the global iteration index $n$, for any $\epsilon > 0$ such that $\mu + \epsilon < 0$, the following holds for all sufficiently large $n$:
\begin{equation} \label{eq-prf-shad-rvi2}
  \big| T_n(s,a) - t_{sa} \big| \leq e^{(\mu + \epsilon) \sum_{k=0}^{n} \beta_{\nu(k, (s,a))}\ind \{ (s,a)  \,\in  \,Y_k\}}, \quad \forall \, (s,a) \in \S \times \A.
\end{equation}
Combining \eqref{eq-prf-shad-rvi2} with \eqref{eq-prf-shad-rvi0} yields that, for all sufficiently large $n$,
\begin{equation} 
 \delta_{n+1} \leq c e^{(\mu + \epsilon) \min_{(s,a) \in \S \times \A} \sum_{k=0}^{n} \beta_{\nu(k, (s,a))}\ind \{ (s,a) \, \in  \,Y_k\}}. \notag
\end{equation} 
Consequently,
\begin{equation} \label{eq-prf-shad-rvi3}
 \limsup_{n \to \infty} \tfrac{\ln \delta_{n+1} }{\sum_{k=0}^n \alpha_k} \leq (\mu + \epsilon) \cdot \liminf_{n \to \infty} \min_{(s,a) \in \S \times \A} \tfrac{\sum_{k=0}^{n} \beta_{\nu(k, (s,a))}\ind \{ (s,a) \, \in  \,Y_k\}}{\sum_{k=0}^n \alpha_k}.
\end{equation} 

We now bound the r.h.s.\ of \eqref{eq-prf-shad-rvi3}. By Assum.~\ref{cond-rvi-alg2}(ii), there exists $\bar k \geq 0$ such that $\beta_k \geq \varsigma \alpha_k$ for all $k \geq \bar k$. Let $C\= \sum_{k=0}^{\bar k -1} \varsigma \alpha_k$, a deterministic constant independent of $n$.
Then, for each $(s,a) \in \S \times \A$, 
\begin{align}
\sum_{k=0}^{n} \beta_{\nu(k, (s,a))} \ind \{ (s,a) \, \in  \,Y_k\} 
\geq \, \varsigma \sum_{k=0}^{n} \alpha_{\nu(k, (s,a))} \ind \{ (s,a) \, \in  \,Y_k\} - C.  \label{eq-prf-shad-rvi3a} 
\end{align}
By Assums.~\ref{cond-rvi-alg}(i),~\ref{cond-ss}(iii) and~\ref{cond-us}(i), we also have 
$$\tfrac{\sum_{k=0}^{n} \alpha_{\nu(k, (s,a))}\ind \{ (s,a) \, \in \, Y_k\}}{\sum_{k=0}^{n} \alpha_k} \to 1 \ \ \  \text{as $n \to \infty$}, \quad \forall \, (s,a) \in \S \times \A.$$
Together with \eqref{eq-prf-shad-rvi3a}, this implies
\begin{equation} \label{eq-prf-shad-rvi3b}
\liminf_{n \to \infty} \tfrac{\sum_{k=0}^{n} \beta_{\nu(k, (s,a))}\ind \{ (s,a) \, \in  \,Y_k\}}{\sum_{k=0}^n \alpha_k} \geq \varsigma, \quad \forall \, (s, a) \in \S \times \A.
\end{equation}
Combining \eqref{eq-prf-shad-rvi3b} with \eqref{eq-prf-shad-rvi3}, and recalling that
$\mu+\epsilon<0$, yields
$\limsup_{n \to \infty} \tfrac{\ln(\delta_{n+1})}{\sum_{k=0}^n \alpha_k} \leq \varsigma (\mu + \epsilon).$
Letting $\epsilon \downarrow 0$ establishes the lemma.  
\end{proof}

By Lem.~\ref{lem-shad-rvi}, we can set $\mu_\delta = \tm \cdot \varsigma \max \big\{ \tfrac{\ell(\{\beta_k\})}{2}, \, - 1 \big\}$ in condition \eqref{eq-prf-shad-rvi0b}; that is, Assum.~\ref{cond-mns} holds here with this choice of $\mu_\delta$.
Since $L_h \leq 2 +  \tm L_f$ by \eqref{eq-prf-shad-rvi0a}, the condition $\mu_\delta < - L_h$ required by Thm.~\ref{thm-3} is satisfied whenever 
$$\tm \cdot \varsigma \max \big\{ \tfrac{\ell(\{\beta_k\})}{2}, \, - 1 \big\} < - (2 +  \tm L_f),$$ 
or equivalently,
\begin{equation} \label{eq-prf-shad-rvi4}
    \varsigma >  \tfrac{2}{\tm} + L_f, \qquad  - \tfrac{\varsigma \, \ell(\{\beta_k\})}{2} > \tfrac{2}{\tm} + L_f.
\end{equation}
These conditions on $\varsigma$ and $\{\beta_k\}$ hold by Assum.~\ref{cond-rvi-alg2}(ii), where the quantity $A_*$ is precisely $\tfrac{2}{\tm} + L_f$. Thus, $\mu_\delta < - L_h$, as required by Thm.~\ref{thm-3}.

Finally, we verify the remaining conditions of Thm.~\ref{thm-3} on the stepsize scaling parameter, which corresponds to $\tm A$ here (cf.\ \eqref{eq-prf-shad-rvi0a}). These conditions require $\tm A > L_h$ for class-2 stepsizes and $\big(\gamma \wedge \tfrac{1}{2}\big) \tm A > L_h$ for class-1 stepsizes, where $\gamma > 0$ is the constant appearing in Assums.~\ref{cond-mus} and~\ref{cond-rvi-alg2}(i) on asynchrony. Since $L_h \leq 2 +  \tm L_f$, these conditions are satisfied if
\[  
\text{for class-2 stepsizes:} \ \   A > \tfrac{2}{\tm} + L_f; \qquad  
\text{for class-1 stepsizes:} \ \   \big(\gamma \wedge \tfrac{1}{2}\big) A > \tfrac{2}{\tm} + L_f.
\]
Both inequalities hold by the assumptions of Thm.~\ref{thm-rvi-ql2}. Hence, all the conditions of Thm.~\ref{thm-3} are satisfied, and the theorem applies, establishing Thm.~\ref{thm-rvi-ql2}. This completes the proof.

\section{Discussion}  \label{sec-conc-rmks}

In this paper, we developed and analyzed a generalized RVI Q-learning algorithm for average-reward RL in weakly communicating SMDPs, proving its convergence using asynchronous SA results from our companion paper \cite{YWS25a}. To fully leverage the underlying SA theory, we equipped RVI Q-learning with new monotonicity conditions for estimating the optimal reward rate, thereby substantially broadening the algorithmic framework and generalizing previous convergence analyses. These RVI Q-learning results have facilitated the development of model-free average-reward RL algorithms for general multichain problems \cite{YuH26}. 

A natural direction for future work is to extend the RVI Q-learning and asynchronous SA framework developed in this paper and its companion paper to distributed computation settings with communication delays.

\section*{Acknowledgments} 
We thank Prof.\ Eugene Feinberg for helpful discussions on average-reward SMDPs; Dr.\ Martha Steenstrup for critical feedback on parts of our earlier draft; and Armanjot Kaur for a question about the finite-variance assumption on the SMDP model, which prompted us to revisit its role in the analysis and clarify it in Rem.~\ref{rem-relax-variance-cond}. We also thank an anonymous reviewer of an earlier version of this work for valuable insights and detailed comments---particularly regarding the use of SMDPs for approximating continuous-time control problems, asymptotically autonomous systems, and the potential for sharper convergence results. The latter suggestion prompted us to study shadowing properties of asynchronous SA algorithms, which ultimately led to the sharper result presented in Thm.~\ref{thm-rvi-ql2} for RVI Q-learning. 

In preparing this manuscript, we used OpenAI ChatGPT (GPT-4 and GPT-5) to refine the writing style.

\addcontentsline{toc}{section}{References} 
\bibliographystyle{apa} 
\let\oldbibliography\thebibliography
\renewcommand{\thebibliography}[1]{%
  \oldbibliography{#1}%
  \setlength{\itemsep}{0pt}%
}
{\fontsize{9}{11} \selectfont
\bibliography{asyn_sa_arXiv3.bib}}

\end{document}